\documentclass[article]{elsarticle}

\usepackage{hyperref}

\journal{Journal of Approximate Reasoning}

\usepackage[utf8]{inputenc}
\usepackage{graphicx}
\usepackage{tikz}
\usepackage{amsmath}
\usepackage{amssymb}
\usepackage{mathtools}
\usepackage{csquotes}
\usepackage{centernot}
\usepackage{enumerate}
\usepackage{cleveref}
\usepackage{subfig}
\usepackage{soul}

\usepackage{amsthm}
\usepackage{algorithm}
\usepackage[noend]{algpseudocode}

\usetikzlibrary{arrows}

\usetikzlibrary{patterns}

\def\hb{\hbox to 10.7 cm{}}
\newtheoremstyle{break}
  {\topsep}{\topsep}%
  {\itshape}{}%
  {\bfseries}{}%
  {\newline}{}%
\theoremstyle{break}
\usepackage[inline, shortlabels]{enumitem}  
\setlist{leftmargin=*}

\newtheorem{proposition}{Proposition}

\newtheorem{definition}{Definition}
\newtheorem{lemma}{Lemma}
\newtheorem{corollary}{Corollary}
\newtheorem{principle}{Principle}

\usepackage[normalem]{ulem}

\makeatletter
\renewcommand{\p@subfigure}{}
\makeatother

\def\QBAG{\mbox{\ensuremath{(\Args, \is, \att, \supp)}}}
\def\QBAF{\QBAG}
\def\QBAFF{\mbox{\ensuremath{(\Args', \is', \att', \supp')}}}
\def\QBAFSTAR{\mbox{\ensuremath{(\Args^*, \is^*, \att^*, \supp^*)}}}

\def\Args{\ensuremath{\mathit{Args}}} 
\def\Att{\ensuremath{\mathit{Att}}}
\def\Supp{\ensuremath{\mathit{Supp}}}
\def\is{\ensuremath{\tau}}
\def\fs{\ensuremath{\sigma}}
\def\att{\Att}
\def\supp{\Supp}

\def\interval{\ensuremath{\mathbb{I}}}

\def\arga{\ensuremath{\mathsf{a}}}
\def\argb{\ensuremath{\mathsf{b}}}
\def\argc{\ensuremath{\mathsf{c}}}
\def\argd{\ensuremath{\mathsf{d}}}
\def\arge{\ensuremath{\mathsf{e}}}
\def\argf{\ensuremath{\mathsf{f}}}
\def\argx{\ensuremath{\mathsf{x}}}
\def\argy{\ensuremath{\mathsf{y}}}
\def\argz{\ensuremath{\mathsf{z}}}

\newcommand{\argnode}[3]{\mbox{\ensuremath{#1~(#2)\!:\!\mathbf{#3}}}}

\usepackage{xifthen}
\usepackage{xargs}
\newcommandx{\scon}[4][3=, 4=]{  
	\relax
	\ifmmode
		\ifthenelse{\isempty {#3}}
		{ {\ensuremath{#1\ \sim #2}} }
		{ {\ensuremath{#1\ \sim_{\sigma,#3,#4} #2}} }
	\else
		\ifthenelse{\isempty {#3}}
		{ {\ensuremath{#1\ \sim #2}} }
		{ {\ensuremath{#1\ \sim_{\sigma,#3,#4} #2}} }
	\fi
						     }  
\newcommandx{\nscon}[4][3=, 4=]{
	\relax
	\ifmmode
		\ifthenelse{\isempty {#3}}
		{ {\ensuremath{#1\ \not\sim #2}} }
		{ {\ensuremath{#1\ \not\sim_{\sigma,#3,#4} #2}} }
	\else
		\ifthenelse{\isempty {#3}}
		{ {\ensuremath{#1\ \not\sim #2}} }
		{ {\ensuremath{#1\ \not\sim_{\sigma,#3,#4} #2}} }
	\fi
						      }  

\newcommandx{\comp}[3][3=]{
	\relax
	\ifmmode
		\ifthenelse{\isempty {#3}}
		{ {\ensuremath{#1\ \lesseqgtr #2}} }
		{ {\ensuremath{#1\ \lesseqgtr_{\sigma,#3} #2}} }
	\else
		\ifthenelse{\isempty {#3}}
		{ {\ensuremath{#1\ \lesseqgtr #2}} }
		{ {\ensuremath{#1\ \lesseqgtr_{\sigma,#3} #2}} }
	\fi
						      }

\newcommandx{\ncomp}[3][3=]{
	\relax
	\ifmmode
		\ifthenelse{\isempty {#3}}
		{ {\ensuremath{#1\ \not\lesseqgtr #2}} }
		{ {\ensuremath{#1\ \not\lesseqgtr_{\sigma,#3} #2}} }
	\else
		\ifthenelse{\isempty {#3}}
		{ {\ensuremath{#1\ \not\lesseqgtr #2}} }
		{ {\ensuremath{#1\ \not\lesseqgtr_{\sigma,#3} #2}} }
	\fi
						      } 

\def\scons{strength-consistent}
\def\nscons{strength-inconsistent}
\def\sconsy{strength consistency}
\def\nsconsy{strength inconsistency}
\newtheorem{example}{Example}
\newtheorem{excont}{Example}

\begin{document}
\begin{frontmatter}
\title{Change in Quantitative Bipolar Argumentation: Sufficient, Necessary, and Counterfactual Explanations} %

\author{Timotheus Kampik}
\ead[url]{tkampik@cs.umu.se}
\address{Umeå University, Sweden}
\address{SAP Signavio, Germany}

\author{Kristijonas \v{C}yras}
\ead[url]{kristijonas.cyras@ericsson.com}
\address{Ericsson, USA}

\author{José Ruiz Alarcón}
\ead[url]{mai21jrn@cs.umu.se}
\address{Umeå University, Sweden}

\begin{abstract}
This paper presents a formal approach to explaining change of inference in Quantitative Bipolar Argumentation Frameworks (QBAFs). When drawing conclusions from a QBAF and updating the QBAF to then again draw conclusions (and so on), our approach traces changes -- which we call \emph{strength inconsistencies} -- in the partial order over argument strengths that a semantics establishes on some arguments of interest, called \emph{topic arguments}. We trace the causes of strength inconsistencies to specific arguments, which then serve as explanations. 
We identify sufficient, necessary, and counterfactual explanations for strength inconsistencies and show that \nsconsy\ explanations exist if and only if an update leads to \nsconsy.
We define a heuristic-based approach to facilitate the search for \nsconsy\ explanations, for which we also provide an implementation.
\end{abstract}


\end{frontmatter}

%
\section{Introduction}\label{intro}
A key challenge in the domain of eXplainable Artificial Intelligence (XAI) is the explanation of an agent's \emph{change of mind}: if the agent has inferred a set of decisions $A$ at time $t_0$, why does it infer another set of decisions $A'$ at $t_1 > t_0$?
This challenge is at the core of fundamental approaches to decision-making and reasoning.
From the perspective of microeconomic decision theory (see, e.g.~\cite{osbournerubinsteintextbook}), a basic assumption is that the agent has \emph{consistent preferences}.
Assuming two independent choices $A$ and $A'$, the agent must not decide $A$ and then change its mind to $A'$ if $A'$ has been available all along and $A$ is still available, all other relevant circumstances being equal.
If an agent's preferences on the available decision options are not consistent, one would expect an explanation that highlights this relevant change in circumstances that violates the \emph{ceteris paribus}\footnote{Translates to: ``all else unchanged'' and is a crucial assumption in classical models of economic rationality.} condition. 
In the parlance of automated reasoning, one would expect that an explanation is provided if monotony of entailment is violated: if the agent first infers $A$ and then $A'$, such that $A \not \subseteq A'$, an explanation of why $A$ is to be rejected should be provided.

One collection of approaches to automated reasoning that has in the past decades gained substantial research interest in the knowledge representation and reasoning community is \emph{formal argumentation}~\cite{Bench-Capon:2007:AAI:1284911.1285088,Baroni2018-BARHOF}.
Typically, formal argumentation approaches draw inferences from graphs representing the relationships between so-called \emph{arguments}, which may, for example, model logical statements.
From an application perspective, formal argumentation is still nascent; it is typically not (or not yet) deployed in large-scale applications.
Yet, because of its somewhat intuitive nature, there are hopes that formal argumentation can bridge the gap between symbolic and subsymbolic AI and facilitate interactive human-machine reasoning~\cite{DBLP:journals/frai/DietzKM22}.
Intuitively, argumentation is inherently dynamic and formally often modelled as an \emph{argumentation dialogue}, where new arguments are added to the argumentation graph over time.
Also, the ``weighting'' of arguments that may correspond to different choice or inference options reflects the dynamic and non-monotonic nature of argumentation-based reasoning.
As a consequence, changes to the relative strengths of arguments during an argumentation process or dialogue may lead to a change in conclusions drawn or actions taken.
Hence, we consider the study of change explanations in formal argumentation as a worthwhile pursuit.

In this paper, we define change explanations in the setting of evolving Quantitative Bipolar Argumentation Frameworks (QBAFs)~\cite{Baroni:Rago:Toni:2019}, i.e.\ for QBAFs from which we draw inferences, then apply changes, in order to draw new inferences from the updated QBAF.
In quantitative bipolar argumentation, arguments are assigned initial strengths, i.e.\ they are nodes with (typically numerical) weights.
Two binary relations on the nodes in the argumentation graph model attacks and supports, respectively, between arguments.
An \emph{argumentation semantics} (inference function) is applied to the graph to infer the \emph{final strengths} of the arguments, considering initial strengths and graph topology.
Recently, the research interest in quantitative (bipolar) argumentation has increased, in particular because of its potential feasibility in application scenarios such as explainable recommendation systems~\cite{DBLP:conf/ijcai/RagoCT18} and neural network explainability~\cite{DBLP:conf/aiia/Sukpanichnant0L21}. 
Hence, we focus on quantitative bipolar argumentation in this paper. 
However, since real-life application scenarios of QBAFs are still nascent, we do not (or not yet) see means for evaluating the QBAF explainability approach that we present in this paper in real-life applications. Instead, we supply theoretical analyses illustrated with simplistic examples, alongside a computational implementation, of a novel form of explanations in QBAFs.

Our goal is to explain the relative change in strength of specified arguments in a QBAF that is updated by changing its arguments, their initial strengths and/or relationships. 
We strive for explanations of changes in arguments' relative strengths that pertain to minimal information causing those changes. 
We adopt the notions of (attributive) sufficient, necessary, and counterfactual explanations\footnote{See e.g.\ \cite{Darwiche_Ji_2022,10.1007/978-3-030-86772-0_4} for notions of sufficient and necessary explanations; see \cite{Stepin.et.al:2021} for an excellent overview of counterfactual explanations.} to the setting of explaining changes in the partial ordering of argument strengths in evolving QBAFs.
In the example below, we give intuitive readings of the introduced concepts; rigorous definitions follow later.
\begin{example}
\label{ex:intro}

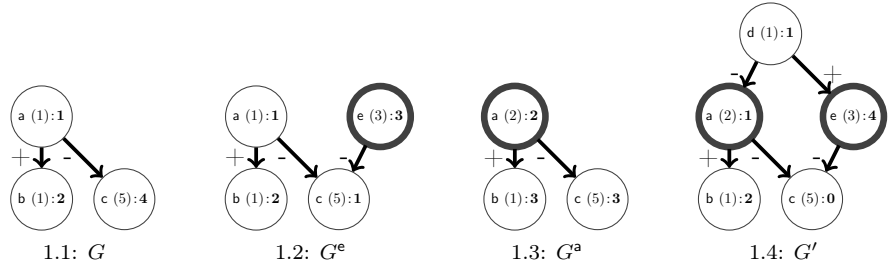
\begin{figure}[!ht]
    \subfloat[$G$]{
    \label{fig:intro:G}
    \begin{tikzpicture}[scale=0.55,
        noanode/.style={scale=0.55,dashed, circle, draw=black!60, minimum size=10mm, font=\bfseries},
        unanode/.style={scale=0.55,circle, draw=black!75, minimum size=10mm, font=\bfseries},
        invnode/.style={scale=0.55,circle, draw=white!0, minimum size=0mm, font=\bfseries},
        anode/.style={scale=0.55,circle, fill=lightgray, draw=black!60, minimum size=10mm, font=\bfseries},
        ]
        \node[unanode]    (a)    at(0,2)  {\argnode{\arga}{1}{1}};
        \node[unanode]  (b)    at(0,0)  {\argnode{\argb}{1}{2}};
        \node[unanode]    (c)    at(2,0)  {\argnode{\argc}{5}{4}};
        \path [->, line width=0.5mm]  (a) edge node[left] {+} (b);
        \path [->, line width=0.5mm]  (a) edge node[left] {-} (c);
    \end{tikzpicture}
    }
    \hspace{10pt}
    \centering
    \subfloat[$G^{\arge}$]{
    \label{fig:intro:Ge}
    \begin{tikzpicture}[scale=0.55,
        noanode/.style={scale=0.55,dashed, circle, draw=black!60, minimum size=10mm, font=\bfseries},
        unanode/.style={scale=0.55,circle, draw=black!75, minimum size=10mm, font=\bfseries},
        xunanode/.style={scale=0.55,circle, draw=black!75, minimum size=10mm, font=\bfseries, line width=2.5pt},
        invnode/.style={scale=0.55,circle, draw=white!0, minimum size=0mm, font=\bfseries},
        anode/.style={scale=0.55,circle, fill=lightgray, draw=black!60, minimum size=10mm, font=\bfseries},
        xanode/.style={scale=0.55, circle, fill=lightgray, draw=black!60, minimum size=10mm, font=\bfseries,line width=2.5pt},
        xnoanode/.style={scale=0.55, circle, dashed, draw=black!60, minimum size=10mm, font=\bfseries,line width=2.5pt}
        ]
        \node[unanode]    (a)    at(0,2)  {\argnode{\arga}{1}{1}};
        \node[unanode]  (b)    at(0,0)  {\argnode{\argb}{1}{2}};
        \node[unanode]    (c)    at(2,0)  {\argnode{\argc}{5}{1}};
        \node[xunanode]    (e)    at(3,2)  {\argnode{\arge}{3}{3}};
        
        \path [->, line width=0.5mm]  (a) edge node[left] {+} (b);
        \path [->, line width=0.5mm]  (a) edge node[left] {-} (c);
        \path [->, line width=0.5mm]  (e) edge node[left] {-} (c);
    \end{tikzpicture}
    }
    \hspace{10pt}
    \centering
    \subfloat[$G^{\arga}$]{
    \label{fig:intro:Ga}
    \begin{tikzpicture}[scale=0.55,
        noanode/.style={scale=0.55,dashed, circle, draw=black!60, minimum size=10mm, font=\bfseries},
        unanode/.style={scale=0.55,circle, draw=black!75, minimum size=10mm, font=\bfseries},
        xunanode/.style={scale=0.55,circle, draw=black!75, minimum size=10mm, font=\bfseries, line width=2.5pt},
        invnode/.style={scale=0.55,circle, draw=white!0, minimum size=0mm, font=\bfseries},
        anode/.style={scale=0.55,circle, fill=lightgray, draw=black!60, minimum size=10mm, font=\bfseries},
        xanode/.style={scale=0.55, circle, fill=lightgray, draw=black!60, minimum size=10mm, font=\bfseries,line width=2.5pt},
        xnoanode/.style={scale=0.55, circle, dashed, draw=black!60, minimum size=10mm, font=\bfseries,line width=2.5pt}
        ]
        \node[xunanode]    (a)    at(0,2)  {\argnode{\arga}{2}{2}};
        \node[unanode]  (b)    at(0,0)  {\argnode{\argb}{1}{3}};
        \node[unanode]    (c)    at(2,0)  {\argnode{\argc}{5}{3}};
        
        \path [->, line width=0.5mm]  (a) edge node[left] {+} (b);
        \path [->, line width=0.5mm]  (a) edge node[left] {-} (c);
    \end{tikzpicture}
    }
    \hspace{10pt}
    \centering
    \subfloat[$G'$]{
    \label{fig:intro:G'}
    \begin{tikzpicture}[scale=0.55,
        noanode/.style={scale=0.55,dashed, circle, draw=black!60, minimum size=10mm, font=\bfseries},
        unanode/.style={scale=0.55,circle, draw=black!75, minimum size=10mm, font=\bfseries},
        xunanode/.style={scale=0.55,circle, draw=black!75, minimum size=10mm, font=\bfseries, line width=2.5pt},
        invnode/.style={scale=0.55,circle, draw=white!0, minimum size=0mm, font=\bfseries},
        anode/.style={scale=0.55,circle, fill=lightgray, draw=black!60, minimum size=10mm, font=\bfseries},
        xanode/.style={scale=0.55, circle, fill=lightgray, draw=black!60, minimum size=10mm, font=\bfseries,line width=2.5pt},
        xnoanode/.style={scale=0.55, circle, dashed, draw=black!60, minimum size=10mm, font=\bfseries,line width=2.5pt}
        ]
        \node[xunanode]    (a)    at(0,2)  {\argnode{\arga}{2}{1}};
        \node[unanode]  (b)    at(0,0)  {\argnode{\argb}{1}{2}};
        \node[unanode]    (c)    at(2,0)  {\argnode{\argc}{5}{0}};
        \node[unanode]    (d)    at(1,4)  {\argnode{\argd}{1}{1}};
        \node[xunanode]    (e)    at(3,2)  {\argnode{\arge}{3}{4}};
        
        \path [->, line width=0.5mm]  (a) edge node[left] {+} (b);
        \path [->, line width=0.5mm]  (a) edge node[left] {-} (c);
        \path [->, line width=0.5mm]  (e) edge node[left] {-} (c);
        \path [->, line width=0.5mm]  (d) edge node[left] {-} (a);
        \path [->, line width=0.5mm]  (d) edge node[right] {+} (e);
    \end{tikzpicture}
    } 
\caption{$G$ and different updates thereof. Here and henceforth, a node labelled $\argnode{\argx}{i}{f}$ carries argument $\argx$ with initial strength $\is(\argx) = i$ and final strength $\fs(\argx) = \mathbf{f}$. Edges labelled $+$ and $-$ respectively represent support and attack. Arguments with bold borders are \nsconsy\ explanation arguments, i.e.\ they explain the change in relative strength of the topic arguments $\argb$ and $\argc$.
            }
\label{fig:intro}
\end{figure}

We start with the QBAF depicted in Figure~\ref{fig:intro:G}, which we denote by $G$. We have the nodes (\emph{arguments}) $\arga$ (with initial strength $\is(\arga) = 1$), $\argb$ (with $\is(\argb) = 1$), and $\argc$ (with $\is(\argc) = 5$); $\arga$ supports $\argb$ and attacks $\argc$.
Here, $\argb$ and $\argc$ are \emph{topic} arguments, i.e.\ arguments that we want to weigh against each other: the topic argument with the highest final strength can be considered the most \emph{promising}. 
For example, $\argb$ may represent a buying decision of a particular commodity and $\argc$ the decision to sell it.
Then, $\arga$ could represent an environmental condition that supports (positively influences) buying the commodity, while attacking (negatively influencing) its sale.
The final strength $\fs(\arga)$ is not directly relevant to the decision but impacts the final strength of topic arguments.
Typically, we determine $\fs(\argx)$ of an argument $\argx$ by aggregating the final strengths of its supporters and attackers. 
For instance, we can add to $\is(\argx)$ the final strengths of supporters of $\argx$ and subtract the final strengths of attackers of $\argx$, iteratively, starting with the neither attacked nor supported leaf arguments (whose final strengths equal their initial strengths). 
Here, we get $\fs(\argb)$ by adding $\fs(\arga) = \is(\arga)$ to $\is(\argb)$: $1+1=2$; and $\fs(\argc)$ by subtracting $\fs(\arga)$ from $\is(\argc)$: $5-1=4$. 
Consequently, $\argc$ is the topic argument with the highest final strength (and so represents the inference drawn from $G$).
\end{example}
In this paper, we are interested in \emph{changing} QBAFs, i.e. in QBAFs whose graphs are subject to different updates.
We continue with the example to provide an intuition of how such updates look like, how they affect the inferences that we draw, and how changes in these inferences can potentially be explained.
\begin{excont}[Continued]
Let us assume that our initial QBAF in Figure~\ref{fig:intro:G} receives a knowledge update, in different forms of changes to the QBAF: we give examples of the different resulting situations in Figures~\ref{fig:intro:Ge}, \ref{fig:intro:Ga}, and \ref{fig:intro:G'}.
These updates can be due to changes in the environment that an agent observes, for example due to changes in market conditions that directly or indirectly influence whether, and to what extent, our agent wants to buy or sell (continuing the scenario of commodity buying or selling).
As we will spell out shortly, we determine the final strengths of $\argb$ and $\argc$ (using the same approach as before) in each situation and find that, after any of the updates, $\argb$, rather than $\argc$, is the highest-ranking topic argument. 
We are then interested in explaining why the ranking of $\argb$ relative to $\argc$ has changed.
Below, we consider three different updates.
The first two situations are rather straightforward; the third one is more involved.

\begin{enumerate}[wide, label=({\bf{\roman*}})]
    \item In Figure~\ref{fig:intro:Ge}, the argument $\arge$ with initial strength of $3$ is added to $G$ thus resulting in $G^{\arge}$. The final strengths of $\argb$ and $\argc$ in $G^{\arge}$ then are $2$ and $1$, respectively. 
    Here, the new argument $\arge$ directly decreases the final strength of $\argc$. 
    Intuitively, (the addition of) $\arge$ explains the change in the relative ordering of the final strengths for $\{ \argb, \argc \}$. 
    \item In Figure~\ref{fig:intro:Ga}, the initial strength of $\arga$ changes from $1$ in $G$ to $2$ in the resulting $G^{\arga}$. The final strengths of $\argb$ and $\argc$ then become equal to $3$. 
    Intuitively, (the change in the initial strength of) $\arga$ explains the change in the relative ordering of the final strengths for $\{ \argb, \argc \}$. 
    \item In Figure~\ref{fig:intro:G'}, we have both addition of new arguments $\argd$ and $\arge$ and relationships thereof, as well as a change to the initial strength of $\arga$. In the resulting $G'$ the final strength of $\argb$ is $2$ and that of $\argc$ is $0$. 
    Now, one could say that here all the changes collectively explain the change in the relative ordering of the final strengths for $\{ \argb, \argc \}$.
    \end{enumerate}
\end{excont}
As we can see, we can naively take all arguments that have been changed (including additions and removals) to explain the change in the relative ordering of the final strengths of two topic arguments. However, to explain this change in a more intuitive manner, let us search for sets of arguments that are in some sense \emph{minimal} explanations.
    \begin{excont}[Continued] 
    For instance, the addition of only $\arge$ \emph{suffices} to make $\argb$ stronger than $\argc$, in the absence of other changes: this is the situation in Figure~\ref{fig:intro:Ge}. 
    Additionally, without adding $\arge$ and in the absence of the other changes we would just have $G$ we started with as in Figure~\ref{fig:intro:G} -- this speaks to \emph{minimality} of the addition of $\arge$ in bringing the changes to the final strengths of the topic arguments. 
    We conclude that $\{ \arge \}$ is a minimal sufficient explanation of the change in the relative ordering of the final strengths for $\{ \argb, \argc \}$. 
    
    Similarly, without the addition of $\argd$ and $\arge$, with only the change to 
    $\arga$, we would be in the situation in Figure~\ref{fig:intro:Ga}, where $\argc$ is not stronger than $\argb$. 
    Hence, $\{ \arga \}$ is also a minimal explanation of the change in relative ordering of the final strengths for $\{ \argb, \argc \}$. 
    
    How about other combinations? 
    Without the change to $\arga$ (but with the addition of $\arge$ and $\argd$), we would find $\fs(\arga) = 0$ and thus $\fs(\argb) = 1 = \fs(\argc)$. I.e.\ the relative strengths of $\argb$ and $\argc$ would change from $G$. So, intuitively, $\{ \argd, \arge \}$ also explains the change. But it is not a minimal explanation, because $\{ \arge \}$ is properly contained in $\{ \argd, \arge \}$. 
    On the other hand, without the addition of $\arge$, we would find $\fs(\arga) = \is(\arga) - \fs(\argd) = 2 - \is(\argd) = 1$,  
    and the final strengths of $\argb$ and $\argc$ would be $\fs(\argb) = \is(\argb) + \fs(\arga) = 2$ and $\fs(\argc) = \is(\argc) - \fs(\arga) = 4$, just as in $G$ to begin with. So $\{ \arga, \argd \}$ is not an explanation, for there is no change in the relative strengths of $\argb$ and $\argc$. 
    Similarly, if the addition of $\argd$ was the only change, we would find $\fs(\arga) = 0$ and the final strengths of $\argb$ and $\argc$ equal to their initial strengths. So $\argd$ alone is not an explanation, either. 
    In the end, we have the following explanations as sets of arguments whose change is \emph{sufficient} to entail the change in the relative strengths of the topic arguments \underline{when changes to other arguments do not happen}: $\{ \arga \}$ and $\{ \arge \}$ are two $\subset$-minimal sufficient explanations; $\{ \argd, \arge \}$ is also  sufficient, but not $\subset$-minimally so; and so is $\{ \arga, \argd, \arge \}$.
\end{excont}

In the above example we observed what we will later call sufficient explanations -- intuitively, changes that are sufficient to imply the change of inference (in the relative order of the topic arguments) in the absence of other changes: which updates in the environment are sufficient to lead the agent to buy rather than sell, all else being equal?
Let us continue the example to illustrate what we will call counterfactual explanations -- intuitively, updates that if reversed, restore the initial inference even if other updates do take place: when would the agent revert back to selling even if the rest of the environment changes?

\begin{excont}[Continued] 

    Another kind of explanatory change that we can observe in updating $G$ to $G'$ is a \emph{counterfactual} one. 
    Let us pose the question as to which changes, if reverted back while keeping all the other changes, would annul the relative change in the ranking of the final strengths of $\argb$ and $\argc$ (i.e.\ would restore \sconsy). 
    Observe first that the absence of the changes to $\arga$ would not counterfactually restore \sconsy: as shown above, if the initial strength of $\arga$ were $1$ in $G'$ with both $\argd$ and $\arge$ as they are, we would have $\fs(\argb) = \fs(\argc) = 1$ in $G'$, whereas $\fs(\argb) < \fs(\argc)$ in $G$. 
    Thus, to restore \sconsy, it does not suffice to revert (changes to) $\arga$ while keeping the other changes.
    On the other hand, as likewise shown above, if $\arge$ were absent from $G'$ with $\arga$ and $\argd$ as they are, we would have $\fs(\argb) = 2 < 4 = \fs(\argc)$; so, counterfactually, if the addition of $\arge$ had not taken place, \sconsy\ would not have happened. 
    In other words, it suffices to revert (changes to) $\arge$ to restore \sconsy\ assuming that all other changes take place.
    Thus, $\{ \arge \}$ is a counterfactual explanation: (the change to) $\arge$ both leads to \nsconsy\ on its own and, if reverted, would restore \sconsy\ \underline{while keeping the other changes}.
    In fact, $\{ \arge \}$ is $\subset$-minimally such: when keeping all the other changes, we \emph{must} revert (the addition of) $\arge$ in order to restore \sconsy\ of $\argb$ and $\argc$ in $G'$; otherwise, we were to revert nothing and thus witness \nsconsy\ in $G'$.
    Therefore, $\{\arge\}$ is a $\subset$-minimal counterfactual explanation.
\end{excont}

The above illustrated changes -- sufficient and counterfactual -- which lead to the change of inference. 
Let us also illustrate changes that are necessitated by the change in inference instead: how must the environment have updated given that the agent switches from selling to buying?

\begin{excont}[Continued] 

    Let us now see which changes are actually necessary for, i.e.\ entailed by, the change in the relative strengths of the topic arguments.
    First, changes to neither only $\arga$ nor only $\arge$ could be said to be \emph{necessary}, because changing neither one specifically is needed for \nsconsy\ (precisely because both $\arga$ and $\arge$ are individually sufficient). 
    Clearly from the above, $\argd$ is not necessary either. 
    Instead, collectively $\{ \arga, \arge \}$ can be said to be necessary, as 
    it is needed to make a change with respect to some argument in $\{ \arga, \arge \}$ to explain $\argc$ ceasing to be stronger than $\argb$ when updating $G$ to $G'$. In other words, if both changes with respect to $\arga$ and $\arge$ were absent, $\argc$ would still be stronger than $\argb$. 
    So $\{ \arga, \arge \}$ is a necessary explanation as a set of arguments changes to at least some of which are \emph{needed} to (entailed by) the change in the relative strengths of the topic arguments, \underline{whether or not changes to other arguments happen}.

\end{excont}

Collectively, the above explanations exhibit the following properties.
Arguments in a sufficient explanation are such that it suffices to make changes to these arguments to bring about \nsconsy, even if no other changes that occur in the update of the QBAF are made. Colloquially, it is sufficient to make changes to only these arguments (and safe to ignore the others) to explain \nsconsy. 
Meanwhile, arguments in a minimal counterfactual explanation lead to \nsconsy\ and are such that reverting changes to exactly these arguments (while keeping the other changes that occur in the update of the QBAF) restores \sconsy. Colloquially, making changes to these arguments explains \nsconsy\ while reverting changes to these arguments explains \sconsy. 
Lastly, arguments in a necessary explanation are such that it is necessary to make changes to at least one of these arguments to bring about \nsconsy, whether or not other changes that occur in the update of the QBAF are made. Colloquially, without changes to at least one of these arguments, there would be no \nsconsy\ to explain. 

To summarise, the explanations' objective is explaining any change in the partial order that the assignment of the final strengths establishes on a set of arguments of interest.
This is achieved by identifying arguments whose \emph{change} (addition, removal, or change of initial strength) is pertinent to the change in the partial order of the final strengths.

In what follows we formalise the intuition given above by defining and analysing novel forms of explanations in QBAFs.
We provide the formal preliminaries in Section~\ref{prelim}. We introduce in Section~\ref{main} our formal framework for explaining change of inference in QBAFs. We analyse the properties of our explanations in Section~\ref{analysis}.
An analysis of properties and constraints that aid the search for minimal explanations and hence facilitate efficient implementation is presented in Section~\ref{implementation}; the source code and documentation of an implementation (in C, with Python bindings to facilitate usage in popular data science and machine learning tool chains) is available at \url{http://s.cs.umu.se/t6xfz2}.
In Section~\ref{discussion} we discuss the novelty and relevance of our work in the context of related research\footnote{Note that this paper revises and extends an earlier conference paper~\cite{DBLP:conf/comma/KampikC22}; most notably, the present paper introduces a third type of explanations, namely necessary explanations, deals with additional edge cases, and provides formal analyses that can serve as foundations for explanation generation in applications, as well as a proof-of-concept implementation.}.
Let us highlight that QBAFs are relatively expressive in comparison to abstract argumentation, which is arguably the most prominent argumentation method; we can use this to our advantage and informally demonstrate the straightforward applicability of our approach to abstract argumentation in Subsection~\ref{applicability}.

\section{Preliminaries}
\label{prelim}
This section introduces the formal preliminaries of our work.
Let $\mathbb{I}$ be a set of elements and let $\preceq$ be a preorder on $\mathbb{I}$. 
Typically, $\mathbb{I} = [0, 1]$ is the unit interval. However, in our primary examples we use a simplistic semantics and hence a different interval, i.e.\ $\interval = \mathbb{R}$; in Subsection~\ref{applicability}, we use a totally ordered set with three elements to model a qualitative argumentation approach.
We stipulate that $\preceq \,=\, \leqslant$ is the standard less-than-equal ordering.
A \emph{quantitative bipolar argumentation framework} contains a set of arguments related by binary \emph{attack} and \emph{support} relations, and assigns an \emph{initial strength} in $\mathbb{I}$ to the arguments.
The initial strength can be thought of as initial credence in, or importance of, arguments. 
Typically, the greater the strength in say the unit interval, the more credible or important the argument is.
\begin{definition}[Quantitative Bipolar Argumentation Framework (QBAF)~\cite{Potyka:2019,Baroni:Rago:Toni:2019}]
A \emph{Quantitative Bipolar Argumentation Framework (QBAF)} is a quadruple
$\QBAF$ consisting of  a set of arguments 
$\Args$, an \emph{attack} relation 
$\Att \subseteq \Args \times \Args$, a \emph{support} relation 
$\Supp \subseteq \Args \times \Args$
and a total function 
$\is : \Args \to \interval$ 
that assigns the \emph{initial strength} $\is(\arga)$ to every $\arga \in \Args$.
\end{definition}

Henceforth, we assume as given a fixed but otherwise arbitrary QBAF $G = \QBAF$, unless specified otherwise. 
We also assume that $\Args$ is finite.

Given $\arga \in \Args$, the set $\Att_{G}(\arga) := \{ \argb~|~\argb \in \Args, (\argb, \arga) \in \Att \}$ is the set of attackers of $\arga$ and
each $\argb \in \Att_{G}(\arga)$ is an \emph{attacker} of $\arga$; 
the set $\Supp_{G}(\arga) := \{ \argc~|~\argc \in \Args, (\argc, \arga) \in  \Supp \}$ is the set of supporters of $\arga$ and 
each $\argc \in \Supp_{G}(\arga)$ is a \emph{supporter} of $\arga$. 
We may drop the subscript $\phantom{}_{G}$ when the context is clear. 
We say that for $\arga, \argb \in \Args$, $\argb$ is reachable from $\arga$ (in $G$), denoted $r_{G}(\arga, \argb)$, iff there is a path from $\arga$ to $\argb$ in the directed graph $(\Args, \Att \cup \Supp)$.
A QBAF is acyclic iff none of its arguments is reachable from itself (and cyclic otherwise).
Given $\Args' \subseteq \Args$, we define the \emph{restriction of $G$ to $\Args'$} as the QBAF $G\downarrow_{\Args} := \left(\Args', \tau \cap (\Args' \times \interval), \Att \cap (\Args' \times \Args'), \Supp \cap (\Args' \times \Args') \right)$.

Reasoning in QBAFs amounts to updating the initial strengths of arguments to their final strengths, taking into account the strengths of attackers and supporters. 
Specifically, given a QBAF, a strength function assigns final strengths to arguments in the QBAF. 
Different ways of defining a strength function are called gradual semantics~\cite{Baroni:Rago:Toni:2019,Potyka:2019}.
\begin{definition}[QBAF Semantics and Strength Functions]\label{semantics-strength-functions}
A \emph{gradual semantics} $\fs$ defines for $G = \QBAF$ a \emph{strength function} $\fs_{G} : \Args \to \interval \cup \{ \bot \}$
that assigns the \emph{final strength} $\fs_{G}(\arga)$ to each argument $\arga \in \Args$, where $\bot$ is a reserved symbol meaning `undefined'.
\end{definition}
We may abuse the notation and drop the subscript $\phantom{}_{G}$ so that $\fs$ denotes the strength function, whenever the context is clear.
The (final) strength of an argument can be thought of as its (final) credence or importance. 
Typically, the greater the strength in $\interval$, the more credible or important the argument is. 
In our examples, we use the real line as the interval: $\interval = \mathbb{R}$. 

A gradual semantics can define a strength function as a composition of multivariate real-valued functions that determines the strength of a given argument by aggregating the strengths of its attackers and supporters, taking into account the initial strengths~\cite{Potyka:2019}. 
A strength function so defined is recursive and generally takes iterated updates to produce a sequence of strength vectors, whence the final strengths are defined as the limits (or fixed points) if they exist. 
However, for \emph{acyclic} QBAFs defining a semantics and computing the final strengths can be more straightforward:
in the topological order of an acyclic QBAF as a graph, start with the leaves,\footnote{Here, leaves are nodes without incoming edges.} 
set their final strengths to equal their initial strengths, 
and then iteratively update the strengths of parents whose all children already have final strengths defined. 
For instance, in Figure \ref{fig:intro:G'} from Example \ref{ex:intro}, we can use the function $\fs(\argx) = \is(\argx) + \left( \sum_{\argy \in Supp(\argx)}\fs(\argy) - \sum_{\argz \in Att(\argx)}\fs(\argz) \right)$ 
defined as a composition, namely sum, of the initial strength ($\is(\argx)$) and the difference between the added final strengths of the supporters and the added final strengths of the attackers ($\sum_{\argx \in Supp(\argx)}\fs(\argy) - \sum_{\argz \in Att(\argx)}\fs(\argz)$). 
It gives final strengths of arguments in the topological order of $G'$: 
first $\fs(\argd) = \is(\argd) = 1$, 
then $\fs(\arga) = \is(\arga) - \fs(\argd) = 1$ and $\fs(\arge) = \is(\arge) + \fs(\argd) = 4$, 
and then $\fs(\argb) = \is(\argb) + \fs(\arga) = 2$ and $\fs(\argc) = \is(\argc) - \fs(\arga) - \fs(\arge) = 0$.

We call \emph{aggregation-influence semantics} such a gradual semantics that given an acyclic QBAF, determines a final strength of an argument by recursively \emph{aggregating} the final strengths of its attackers and supporters to then combine with the \emph{influence} of the aggregation\footnote{Mossakowski and Neuhaus introduce these semantics as \emph{modular semantics}~\cite{mossakowski2018modular}; here, we opt for a more descriptive term.}.

\begin{definition}[Aggregation-Influence Semantics]
\label{definition:aggregation-influence}
  A gradual semantics $\sigma$ is an aggregation-influence semantics iff, given an acyclic QBAF $G = \QBAF$ and argument $\argx \in \Args$ it holds that $\fs(\argx) = f_{\is(\argx)}\left(g(\{ \sigma(\argy) | \argy \in \Att(\argx) \}, \{ \sigma(\argy) | \argy \in \Supp(\argx) \})\right)$, 
  where $g: \interval^{|\Att(\argx)|} \times \interval^{|\Supp(\argx)|} \rightarrow \interval$ (aggregation function) and $f_{\is(\argx)}: \interval \rightarrow \interval$ (parameterised influence function).
\end{definition}
Table~\ref{table:semantics} provides a list of common influence and aggregation functions.
Table~\ref{table:semanticsExamples} shows some examples of aggregation-influence semantics.
\begin{table}[ht!]
\footnotesize{
\begin{tabular}{lll}
\hline
\multicolumn{3}{c}{\textbf{Aggregation Functions}} \\ \hline
Sum & $\alpha^{\Sigma}_{v}: [0,1]^n \rightarrow \mathbb{R}$ & $\alpha^{\Sigma}_{v}(s) = \sum_{i = 1}^n v_i \times s_i $  \\
Product & $\alpha^{\Pi}_{v}: [0,1]^n \rightarrow [-1, 1]$ & $\alpha^{\Pi}_{v}(s) = \prod_{i:v_i=-1} (1 - s_i) - \prod_{i:v_i=1} (1 - s_i)$  \\
Top & $\alpha^{max}_{v}: [0,1]^n \rightarrow [-1, 1]$ & $a_v^{max}(s) = M_v(s) - M_{-v}(s),$ \\
 &  & where $M_v(s) = max\{0,v_1 \times s_1, \dots, v_n \times s_n\}$  \\
\hline
\multicolumn{3}{c}{\textbf{Influence Functions}} \\ \hline
Linear($k$) & $\iota^{l}_{w}: [-k, k] \rightarrow [0, 1]$ & $\iota^{l}_{w}(s) = w - \frac{w}{k} \times max\{0,-s\} + \frac{1-w}{k} \times max\{0, s\}$ \\
Euler-based & $\iota^{e}_{w}: \mathbb{R} \rightarrow [w^2, 1]$ & $\iota^{e}_{w}(s) = 1 - \frac{1-w^2}{1 + w \times e^s}$ \\
p-Max($k$) & $\iota^{p}_{w}: \mathbb{R} \rightarrow [0, 1]$ & $\iota^{p}_{w}(s) = w - w \times h(- \frac{s}{k}) + (1-w) \times h(\frac{s}{k}),$  \\
for $p \in \mathbb{N}$ &  & where $h(x) = \frac{max\{0,x\}^p}{1 + max\{0,x\}^p}$  \\
\end{tabular}}
\caption{Common aggregation $\alpha$ and (parametrized) influence $\iota$ functions~\cite[pp.~1724 Table 1; with a fixed typo for p-Max($k$)]{Potyka:2019}. Parameter $s$ represents the strength of an argument at that state, $w$ the initial strength,  and $s_i$ and $v_i \in \{ -1, 1 \}$ the strengths and relationships, respectively, of the argument's attackers/supporters.}
\label{table:semantics}
\end{table}

\begin{table}[ht!]
\centering
\begin{tabular}{lll}
\hline
\textbf{Semantics}           & \textbf{Aggregation} & \textbf{Influence}  \\ \hline
QuadraticEnergyModel         & Sum         & 2-Max(1)  \\
SquaredDFQuADModel & Product     & 1-Max(1)  \\
EulerBasedTopModel & Top         & EulerBased \\
EulerBasedModel    & Sum         & EulerBased \\
DFQuADModel        & Product     & Linear(1) 
\end{tabular}
\caption{Examples of gradual semantics.}
\label{table:semanticsExamples}
\end{table}

While many gradual semantics can be defined for QBAFs in general, their convergence is not always guaranteed in a particular QBAF. 
For several well-studied semantics, convergence is however always guaranteed in acyclic QBAFs\footnote{See e.g.\ \cite{Potyka:2019} for a neat exposition of convergence results under various semantics.}. 
For illustration purposes to avoid dealing with the sometimes demanding definitions of strength functions, we use acyclic QBAFs and the above strength function $\sigma$ (in accordance with a topological ordering of an acyclic QBAF). We however note that both the formal definitions and theoretical analysis of \nsconsy\ explanations given in the paper apply to the general setting of gradual semantics giving total or partial strength functions.
In contrast, we will see that our implementation relies on well-defined semantics that give total strength functions and satisfy \emph{directional connectedness}~\cite{tree} (roughly, that an argument affects the final strength of another only if the latter is reachable from the former; we will define this in Section~\ref{implementation}).

\section{Change Explainability in QBAFs}
\label{main}
In this section, we introduce our formal approach to change explainability in QBAFs.
Henceforth in this and the next section, unless stated otherwise, we let $G = \QBAF$ and $G' = \QBAFF$ be QBAFs, let $\arga, \argb, \argx, \argy \in \Args \cap \Args'$, let $\sigma$ be a strength function, and let $S \subseteq \Args \cup \Args'$. Let us highlight that we do not formalise the QBAF change/update operation; instead, we assume that we have two QBAFs that have at least two arguments in common, and the second QBAF (namely $G'$) can be considered an updated (or evolved) version of the first one (namely $G$).

We start by defining the preliminary notion of \emph{comparability} to then introduce the central definition of \emph{\sconsy}.
We need argument comparability to make sure we are comparing well-defined strengths of arguments.

\begin{definition}[Argument Comparability w.r.t.\ Final Strength]\label{def:comparability}
We say that $\arga$ is \emph{comparable} with $\argb$ (in $G$), denoted by $\comp{\arga}{\argb}[G]$, iff $\fs_{G}(\arga) \neq \bot$ and $\fs_{G}(\argb) \neq \bot$ and ($\fs_{G}(\arga) \geq \fs_{G}(\argb)$ or $\fs_{G}(\arga) \leq \fs_{G}(\argb)$).
\end{definition}
If $\comp{\arga}{\argb}[G]$ does not hold, we call $\arga$ and $\argb$ \emph{incomparable}, denoted by $\ncomp{\arga}{\argb}[G]$, and when there is no ambiguity, we drop the subscripts and write $\comp{\arga}{\argb}$ to denote comparability and $\ncomp{\arga}{\argb}$ to denote incomparability.

Now we can introduce the central notion to qualitatively evaluate final strength changes in evolving QBAFs.
\begin{definition}[Strength Consistency]\label{def:po-consistency}
We say that $\arga$ is \emph{\scons} w.r.t.\ $\argb$, 
denoted by $\scon{\arga}{\argb}[G][G']$, iff the following statements hold true:
\begin{itemize}
    \item If $\fs_{G}(\arga) > \fs_{G}(\argb)$ then $\fs_{G'}(\arga) > \fs_{G'}(\argb)$;
    \item If $\fs_{G}(\arga) < \fs_{G}(\argb)$ then $\fs_{G'}(\arga) < \fs_{G'}(\argb)$;
    \item If $\fs_{G}(\arga) = \fs_{G}(\argb)$ then $\fs_{G'}(\arga) = \fs_{G'}(\argb)$;
    \item $\comp{\arga}{\argb}[G]$ iff $\comp{\arga}{\argb}[G']$.
\end{itemize}
\end{definition}
Intuitively, two arguments are \scons\ just in case their relative strengths correspond between the two QBAFs. 
In an obvious way, $\nscon{\arga}{\argb}[G][G']$ denotes the negation of $\scon{\arga}{\argb}[G][G']$ and we say that $\arga$ and $\argb$ are \emph{\nscons}. 
When there is no ambiguity, we drop the subscripts and write $\arga \sim \argb$ to denote that $\arga$ is \scons\ w.r.t.\ $\argb$, and similarly for the derived notions.

The main objective of this paper is to define explanations as to what changes in an evolving QBAF lead to any two \emph{topic} arguments being \nscons\ (or \scons). 
As illustrated in the \nameref{intro} section, we seek to define explanations of three different types as follows.

\begin{enumerate}
\item Sufficient explanations should capture only those arguments with respect to which the changes in the update of the QBAF entail the same partial order of the final strengths of topic arguments as is obtained after the actual update.

\item 
Counterfactual explanations should capture only those arguments with respect to which changes in the update of the QBAF both (i) entail the same partial order of the final strengths of topic arguments as is obtained after the actual update, and, (ii) if reverted, would restore the order before the update.

\item Necessary explanations should capture all those arguments with respect to which changes in the update of the QBAF are entailed by the change in the partial order of the final strengths of topic arguments obtained after the actual update.
\end{enumerate}

As a prerequisite for defining our explanations, we introduce the notion of a QBAF \emph{reversal} with respect to a set of arguments, where such sets of arguments will later play the role of explanations. 
Colloquially speaking, given QBAFs $G$ and its update $G'$, a reversal of $G'$ to $G$ w.r.t.\ a set of arguments $S$ updates the properties of every argument from $S$ in $G'$ so that they reflect the properties of the same argument in $G$: arguments from $S$ that are not in $G$ are deleted and arguments from $S$ that are in $G$ but not in $G'$ are restored.
\begin{definition}[QBAF Reversal]\label{def:qbaf-reversal}
We define the \emph{reversal} of $G'$ to $G$ w.r.t.\ $S \subseteq \Args \cup Args'$, denoted by $G_{\leftarrow G'}(S)$, as a QBAF $(\Args^*, \tau^*, \Att^*, \Supp^*)$, where:
\begin{itemize}
    \item $\Args^* = \left( \Args' \cup S \right) \setminus (S \setminus \Args)$;
    \item $\Att^* = \left(\underbrace{(\Att'\setminus(S\times\Args))}_{\text{$\Att'$, but not from $S$ to $Args$}} \cup \underbrace{\left((S\times\Args^*)\cap\Att\right)}_{\text{$\Att$ from $S$ to $\Args^*$}}\right) \cap (\Args^* \times \Args^*)$;
    \item $\Supp^* = \Bigl(\left(\Supp'\setminus(S\times\Args)\right) \cup \left((S\times\Args^*)\cap\Supp\right)\Bigr) \cap (\Args^* \times \Args^*)$; 
    \item $\tau^*: \Args^* \rightarrow \mathbb{I}$ and $\forall \argx \in \Args^*$ the following statement holds true:
    \begin{align*}
       &{} \tau^*(\argx) = 
       \begin{cases}
        \tau(\argx), \text{ if } \argx \in \Args \cap S; \\
        \tau'(\argx), \text{ otherwise }.
    \end{cases}
    \end{align*}
\end{itemize}
\end{definition}
Intuitively:
for arguments that were removed (i.e.\ arguments from $\Args \setminus \Args'$), those from $S$ are added back;
for arguments that were added (i.e.\ arguments from $\Args' \setminus \Args$), those from $S$ are removed. 
The arguments are restored with the associated initial strengths, attacks and supports: in the reversal, we restore ``old'' attacks and supports from $S$; 
we leave ``new'' attacks and supports unless they are from $S$ either to the ``old'' arguments or to arguments that have been removed with the reversal. 
For visual intuition, a Venn diagram of the set $\Args^*$ is given in Figure \ref{fig:venn} and two examples of QBAF reversals for QBAFs from Example~\ref{ex:intro} are given in Figure~\ref{fig:reversals-e}. 
\begin{figure}[ht!]
\centering
\begin{tikzpicture}[line width=0.35mm]
  \draw[pattern=north west lines, pattern color=lightgray] (1.5,0) circle (1.5cm) node[right] {\phantom{e} $\Args'$};
  \draw[pattern=north west lines, pattern color=lightgray] (0.75,-0.3) circle (0.9cm) node[above] {$S$};
    \begin{scope}
    \clip (0.75,-0.3) circle (0.865cm);
    \fill[white] (1.5,0) circle (1.5cm);
    \clip (0,0) circle (1.5cm);
    \fill[pattern=north west lines, pattern color=lightgray] (1.5,0) circle (1.5cm);
    \draw (0.75,-0.3) node[above] {$S$};
    \end{scope}
    \draw (0,0) circle (1.5cm) node[left] {$\Args$ \phantom{e}};
\end{tikzpicture}

\caption{Venn diagram for $\Args^* = (\Args' \cup S) \setminus (S \setminus \Args)$ (shaded with lines) in the reversal $G_{\leftarrow G'}(S) = \QBAFSTAR{}$ of $G'$ to $G$ w.r.t.\ $S \subseteq \Args' \cup \Args$.}
\label{fig:venn}
\end{figure}
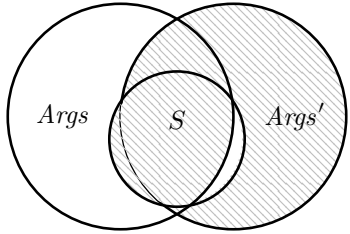

\begin{figure}[!ht]
    \subfloat[$G$]{
    \label{fig:reversal-original}
    \begin{tikzpicture}[scale=0.55,
        unanode/.style={scale=0.55,circle, draw=black!75, minimum size=10mm, font=\bfseries}
        ]
        \node[unanode]    (a)    at(0,2)  {\argnode{\arga}{1}{1}};
        \node[unanode]  (b)    at(0,0)  {\argnode{\argb}{1}{2}};
        \node[unanode]    (c)    at(2,0)  {\argnode{\argc}{5}{4}};
        \path [->, line width=0.5mm]  (a) edge node[left] {+} (b);
        \path [->, line width=0.5mm]  (a) edge node[left] {-} (c);
    \end{tikzpicture}
    }
    \centering
    \subfloat[$G'$]{
    \label{fig:reversal-update}
    \begin{tikzpicture}[scale=0.55,
        unanode/.style={scale=0.55,circle, draw=black!75, minimum size=10mm, font=\bfseries}
        ]
        \node[unanode]    (a)    at(0,2)  {\argnode{\arga}{2}{1}};
        \node[unanode]  (b)    at(0,0)  {\argnode{\argb}{1}{2}};
        \node[unanode]    (c)    at(2,0)  {\argnode{\argc}{5}{0}};
        \node[unanode]    (d)    at(1,4)  {\argnode{\argd}{1}{1}};
        \node[unanode]    (e)    at(3,2)  {\argnode{\arge}{3}{4}};
        
        \path [->, line width=0.5mm]  (a) edge node[left] {+} (b);
        \path [->, line width=0.5mm]  (a) edge node[left] {-} (c);
        \path [->, line width=0.5mm]  (e) edge node[left] {-} (c);
        \path [->, line width=0.5mm]  (d) edge node[left] {-} (a);
        \path [->, line width=0.5mm]  (d) edge node[right] {+} (e);
    \end{tikzpicture}
    }
    \centering
    \subfloat[$G_{\gets G'}(\{\arge\})$]{
    \label{fig:reversal-e}
    \begin{tikzpicture}[scale=0.55,
        old/.style={scale=0.55,circle, draw=black!75, minimum size=10mm, font=\bfseries, line width=1pt},
        restored/.style={scale=0.55,circle, draw=black!65, minimum size=10mm, font=\bfseries, line width=1pt},
        new/.style={scale=0.55,circle, dashed, draw=black!80, minimum size=10mm, font=\bfseries, line width=1pt}
        ]
        \node[new]    (a)    at(0,2)  {\argnode{\arga}{2}{1}};
        \node[old]  (b)    at(0,0)  {\argnode{\argb}{1}{2}};
        \node[old]    (c)    at(2,0)  {\argnode{\argc}{5}{4}};
        \node[new]    (d)    at(1,4)  {\argnode{\argd}{1}{1}};
        
        \node at(3.5,2)  {\phantom{ }};  
        
        \path [->, line width=0.5mm]  (a) edge node[left] {+} (b);
        \path [->, line width=0.5mm]  (a) edge node[left] {-} (c);
        \path [->, line width=0.5mm, dashed]  (d) edge node[left] {-} (a);
    \end{tikzpicture} 
    } 
    \centering
    \subfloat[$G_{\gets G'}(\{ \arga, \argb, \argc, \argd \})$]{
    \label{fig:reversal-but-e}
    \begin{tikzpicture}[scale=0.55,
        old/.style={scale=0.55,circle, draw=black!75, minimum size=10mm, font=\bfseries, line width=1pt},
        restored/.style={scale=0.55,circle,  densely dotted, draw=black!65, minimum size=10mm, font=\bfseries, line width=1pt},
        new/.style={scale=0.55,circle, dashed, draw=black!80, minimum size=10mm, font=\bfseries, line width=1pt}
        ]
        \node[restored]    (a)    at(0,2)  {\argnode{\arga}{1}{1}};
        \node[old]  (b)    at(0,0)  {\argnode{\argb}{1}{2}};
        \node[old]    (c)    at(2,0)  {\argnode{\argc}{5}{1}};
        \node[new]    (e)    at(3,2)  {\argnode{\arge}{3}{3}};

        \node at(5,2)  {\phantom{ }};  
        
        \path [->, line width=0.5mm]  (a) edge node[left] {+} (b);
        \path [->, line width=0.5mm]  (a) edge node[left] {-} (c);
        \path [->, line width=0.5mm, dashed]  (e) edge node[left] {-} (c);
    \end{tikzpicture}
    }
\caption{Using QBAFs $G$ and $G'$ from Example~\ref{ex:intro} (for reference in Figures \ref{fig:reversal-original} and \ref{fig:reversal-update}), we illustrate reversals of $G'$ to $G$ w.r.t.\ sets $\{ \arge \}$ and $\left( \Args \cup \Args' \right) \setminus \{ \arge \} = \{ \arga, \argb, \argc, \argd \}$ (i.e.\ w.r.t.\ $\arge$ and everything except $\arge$) in Figures \ref{fig:reversal-e} and \ref{fig:reversal-but-e}, respectively. Arguments (and relationships, if any) with dashed borders are those that were changed and not reverted; the reverted ones have pointed borders; all others stay unchanged before and after the reversal.}
\label{fig:reversals-e}
\end{figure}
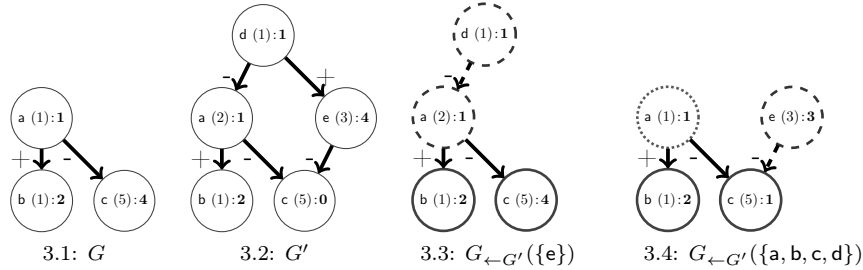

Using the notion of a QBAF reversal, we introduce different notions of \emph{\nsconsy\ explanations}, that are sets of arguments intuitively described as follows.
\begin{itemize}
    \item $\emptyset$ is a sufficient, necessary, and counterfactual explanation 
    if we do not find \nsconsy\ after the update from $G$ to $G'$.
    \item $S \neq \emptyset$ is a sufficient explanation of \nsconsy\ after the update from $G$ to $G'$ if the inconsistency persists when we \emph{revert everything except $S$} back. Note that this formulation coincides with sufficing to change only $S$ to cause inconsistency.
    \item $C \neq \emptyset$ is a counterfactual explanation of \nsconsy\ after the update from $G$ to $G'$ if $C$ suffices for the inconsistency and the inconsistency ceases when we \emph{revert $C$ itself} back while keeping the other changes.
    \item $N \neq \emptyset$ is a necessary explanation of \nsconsy\ after the update from $G$ to $G'$ if $N$ meets every sufficient explanation, i.e.\ no sufficient explanation is disjoint from $N$. Note that this formulation coincides with necessitating to change something in $N$ to cause inconsistency, because inconsistency is caused by some sufficient changes and those are met by $N$.
\end{itemize}
For all three types of explanations, we are mostly interested in their $\subset$-minimal versions\footnote{In particular, to respect the semantics of the word ``necessary'', we will define only minimally necessary explanations and will illustrate thereafter why.}.
\begin{definition}[Strength Inconsistency Explanations]\label{def:inconsistency-explanations}
We say that 
\begin{itemize}
    \item $S \subseteq \Args' \cup \Args$ is a \emph{Sufficient Strength Inconsistency (SSI) explanation} of $\argx$ and $\argy$ w.r.t.\ $\fs$, $G$, $G'$, and a class of QBAFs $\cal G$ iff the following statement holds true:
    \begin{align*}
        &{} \textbf{either}
        \underbrace{\left(
        S = \emptyset \quad\text{ and }\quad \scon{\argx}{\argy}[G][G'] 
        \right)}_{\text{$\argx$ and $\argy$ are \scons, so empty explanation}}\\
        &{} \textbf{or}
        \underbrace{(
        \nscon{\argx}{\argy}[G][G'] \quad\text{ and }\quad 
        \nscon{\argx}{\argy}[G][G_{\leftarrow G'}((\Args \cup \Args')\setminus S)]
        }_{\text{$\argx$ and $\argy$ are \nscons\ and remain so after reversing everything but $S$ back}}\\
        &{} \quad\quad \textbf{and} \quad\quad\:
        \underbrace{G_{\leftarrow G'}((\Args \cup \Args')\setminus S) \in {\cal G}}_{\text{the reversal is in our search space $\cal G$}}
        )
    \end{align*}
    $SX_{{\cal G}}(\nscon{\argx}{\argy}[G][G'])$ denotes the set of all SSI explanations of $\argx$ and $\argy$ w.r.t.\ $\fs$, $G$, $G'$, and $\cal G$
    and $SX_{\subset_{\min},{\cal G}}(\nscon{\argx}{\argy}[G][G'])$ denotes the set of all $\subset$-minimal SSI explanations of $\argx$ and $\argy$ w.r.t.\ $\fs$, $G$, $G'$, and $\cal G$.
    \item $C \subseteq \Args' \cup \Args$ is a \emph{Counterfactual Strength Inconsistency (CSI) explanation} of $\argx$ and $\argy$ w.r.t.\ $\fs$, $G$, $G'$, and a class of QBAFs $\cal G$ iff the following statement holds true:
    \begin{align*} %
        &{} \quad\quad\quad\quad\quad\quad\:\:\: \underbrace{C \in SX_{{\cal G}}(\nscon{\argx}{\argy}[G][G'])}_{\text{$C$ is an SSI of $\argx$ and $\argy$}} \\
        &{} \textbf{and } \underbrace{\scon{\argx}{\argy}[G][G_{\leftarrow G'}(C)]}_{\text{$\argx$ and $\argy$ become \scons\ after reversing $C$}} \\
         &{}\textbf{and } \quad \quad \quad \underbrace{G_{\leftarrow G'}(C) \in {\cal G}}_{\text{the reversal is in our search space $\cal G$}} 
    \end{align*} 
    $CX_{{\cal G}}(\nscon{\argx}{\argy}[G][G'])$ denotes the set of all CSI explanations of $\argx$ and $\argy$ w.r.t.\ $\fs$, $G$, $G'$, and $\cal G$ and 
    $CX_{\subset_{\min},{\cal G}}(\nscon{\argx}{\argy}[G][G'])$ denotes the set of all $\subset$-minimal CSI explanations of $\argx$ and $\argy$ w.r.t.\ $\fs$, $G$, $G'$, and $\cal G$.
    \item $N \subseteq \Args' \cup \Args$ is a \emph{Necessary Strength Inconsistency (NSI) explanation} of $\argx$ and $\argy$ w.r.t.\ $\sigma$, $G$, $G'$, and a class of QBAFs $\cal G$ iff the following statement holds true:
    \begin{align*}
        &{} \quad\quad\quad\quad\quad\quad\quad\quad\:\: \underbrace{N \in SX_{{\cal G}}(\nscon{\argx}{\argy}[G][G'])}_{\text{$N$ is an SSI of $\argx$ and $\argy$}} \\
        &{} \textbf{and } \underbrace{\nexists S \subseteq (\Args \cup \Args') \setminus N \text{ s.t.\ } S \in SX_{{\cal G}}(\nscon{\argx}{\argy}[G][G'])}_{\text{no $S$ disjoint from $N$ is an SSI of $\argx$ and $\argy$}} \\
        &{} \textbf{and } N \text{ is $\subset$-minimal such set.}
    \end{align*}
    $NX_{\subset_{\min},{\cal G}}(\nscon{\argx}{\argy}[G][G'])$ denotes the set of all ($\subset$-minimal) NSI explanations of $\argx$ and $\argy$ w.r.t.\ $\fs$, $G$, $G'$, and $\cal G$.
\end{itemize}
Analogously to the case of \sconsy, when there is no ambiguity, we may drop the subscripts and write simply $SX(\argx \not \sim \argy)$ to denote all SSI explanations of $\argx$ and $\argy$ (w.r.t.\ the implicit $\fs$, $G$, $G'$, and ${\cal G}$), and similarly for the derived notions; then, we assume by default the implicit ${\cal G}$ is the class of all QBAFs.
\end{definition}
We explain further below (Example~\ref{ex:cycles}) why making the explanations dependent on a class of QBAFs is of practical relevance. Intuitively, changes to arguments in a $\subset$-minimal SSI explanation are minimally sufficient to explain \nsconsy\ even if other changes do not happen. 
Changes to arguments in a CSI explanation $C$ are also sufficient but need not be necessary for \nsconsy\ to take place, because there may be some combination (i.e.\ an SSI explanation) $S$ of other changes that result in \nsconsy; instead, the absence of changes to arguments in $C$ while keeping all other changes would restore \sconsy. 
Finally, changes to at least some arguments in an NSI explanation $N$ are needed, because $N$ meets every SSI explanation $S$, and so, contrapositively, if no changes to arguments in $N$ were made, then, in particular, no minimal SSI explanation $S$ would have all of its arguments changed, and so changes to at least some arguments would be lacking to bring about the \nsconsy. 

Let us revisit the example from the \emph{\nameref{intro}} section to illustrate how \nsconsy\ explanations explain change of inference in QBAFs, this time with formal notation.
\begin{example}[Example~\ref{ex:intro} revisited]
\label{ex:inconsistency-explanations}
\begin{figure}[!ht]  
    \subfloat[$G$]{
    \label{fig:SSI-example:G}\label{fig:example:G}
    \begin{tikzpicture}[scale=0.55,
        noanode/.style={scale=0.55,dashed, circle, draw=black!60, minimum size=10mm, font=\bfseries},
        unanode/.style={scale=0.55,circle, draw=black!75, minimum size=10mm, font=\bfseries},
        invnode/.style={scale=0.55,circle, draw=white!0, minimum size=0mm, font=\bfseries},
        anode/.style={scale=0.55,circle, fill=lightgray, draw=black!60, minimum size=10mm, font=\bfseries},
        ]
        \node[unanode]    (a)    at(0,2)  {\argnode{\arga}{1}{1}};
        \node[unanode]  (b)    at(0,0)  {\argnode{\argb}{1}{2}};
        \node[unanode]    (c)    at(2,0)  {\argnode{\argc}{5}{4}};
        \path [->, line width=0.5mm]  (a) edge node[left] {+} (b);
        \path [->, line width=0.5mm]  (a) edge node[left] {-} (c);
    \end{tikzpicture}
    }
    \hspace{20pt}
    \centering
    \subfloat[$G'$]{
    \label{fig:SSI-example:G'}\label{fig:example:G'}
    \begin{tikzpicture}[scale=0.55,
        noanode/.style={scale=0.55,dashed, circle, draw=black!60, minimum size=10mm, font=\bfseries},
        unanode/.style={scale=0.55,circle, draw=black!75, minimum size=10mm, font=\bfseries},
        xunanode/.style={scale=0.55,circle, draw=black!75, minimum size=10mm, font=\bfseries, line width=2.5pt},
        invnode/.style={scale=0.55,circle, draw=white!0, minimum size=0mm, font=\bfseries},
        anode/.style={scale=0.55,circle, fill=lightgray, draw=black!60, minimum size=10mm, font=\bfseries},
        xanode/.style={scale=0.55, circle, fill=lightgray, draw=black!60, minimum size=10mm, font=\bfseries,line width=2.5pt},
        xnoanode/.style={scale=0.55, circle, dashed, draw=black!60, minimum size=10mm, font=\bfseries,line width=2.5pt}
        ]
        \node[xunanode]    (a)    at(0,2)  {\argnode{\arga}{2}{1}};
        \node[unanode]  (b)    at(0,0)  {\argnode{\argb}{1}{2}};
        \node[unanode]    (c)    at(2,0)  {\argnode{\argc}{5}{0}};
        \node[unanode]    (d)    at(1,4)  {\argnode{\argd}{1}{1}};
        \node[xunanode]    (e)    at(3,2)  {\argnode{\arge}{3}{4}};
        
        \path [->, line width=0.5mm]  (a) edge node[left] {+} (b);
        \path [->, line width=0.5mm]  (a) edge node[left] {-} (c);
        \path [->, line width=0.5mm]  (e) edge node[left] {-} (c);
        \path [->, line width=0.5mm]  (d) edge node[left] {-} (a);
        \path [->, line width=0.5mm]  (d) edge node[right] {+} (e);
    \end{tikzpicture} 
    } 
    \hspace{20pt}
    \centering
    \subfloat[$G^{\arga} = G_{\gets G'}(S_{\{ \arga \}})$]{
    \label{fig:example:Ga}
    \begin{tikzpicture}[scale=0.55,
        noanode/.style={scale=0.55,dashed, circle, draw=black!60, minimum size=10mm, font=\bfseries},
        unanode/.style={scale=0.55,circle, draw=black!75, minimum size=10mm, font=\bfseries},
        xunanode/.style={scale=0.55,circle, draw=black!75, minimum size=10mm, font=\bfseries, line width=2.5pt},
        invnode/.style={scale=0.55,circle, draw=white!0, minimum size=0mm, font=\bfseries},
        anode/.style={scale=0.55,circle, fill=lightgray, draw=black!60, minimum size=10mm, font=\bfseries},
        xanode/.style={scale=0.55, circle, fill=lightgray, draw=black!60, minimum size=10mm, font=\bfseries,line width=2.5pt},
        xnoanode/.style={scale=0.55, circle, dashed, draw=black!60, minimum size=10mm, font=\bfseries,line width=2.5pt}
        ]
        \node[xunanode]    (a)    at(0,2)  {\argnode{\arga}{2}{2}};
        \node[unanode]  (b)    at(0,0)  {\argnode{\argb}{1}{3}};
        \node[unanode]    (c)    at(2,0)  {\argnode{\argc}{5}{3}};
        \node at(6,2)  {\phantom{ }};  
        \path [->, line width=0.5mm]  (a) edge node[left] {+} (b);
        \path [->, line width=0.5mm]  (a) edge node[left] {-} (c);
    \end{tikzpicture}
    }
    \\
    \centering
    \subfloat[$G^{*} \newline\phantom{ }\qquad= G_{\gets G'}(\{\arga\})$]{
    \label{fig:SSI-example:G-reversal2}\label{fig:example:G*}
    \begin{tikzpicture}[scale=0.55,
        noanode/.style={scale=0.55,dashed, circle, draw=black!60, minimum size=10mm, font=\bfseries},
        unanode/.style={scale=0.55,circle, draw=black!75, minimum size=10mm, font=\bfseries},
        xunanode/.style={scale=0.55,circle, draw=black!75, minimum size=10mm, font=\bfseries, line width=2.5pt},
        invnode/.style={scale=0.55,circle, draw=white!0, minimum size=0mm, font=\bfseries},
        anode/.style={scale=0.55,circle, fill=lightgray, draw=black!60, minimum size=10mm, font=\bfseries},
        xanode/.style={scale=0.55, circle, fill=lightgray, draw=black!60, minimum size=10mm, font=\bfseries,line width=2.5pt},
        xnoanode/.style={scale=0.55, circle, dashed, draw=black!60, minimum size=10mm, font=\bfseries,line width=2.5pt}
        ]
        \node[xunanode]    (a)    at(0,2)  {\argnode{\arga}{1}{0}};
        \node[unanode]  (b)    at(0,0)  {\argnode{\argb}{1}{1}};
        \node[unanode]    (c)    at(2,0)  {\argnode{\argc}{5}{1}};
        \node[unanode]    (d)    at(1,4)  {\argnode{\argd}{1}{1}};
        \node[unanode]    (e)    at(3,2)  {\argnode{\arge}{3}{4}};
        \node at(5,2)  {\phantom{ }};  
        \path [->, line width=0.5mm]  (a) edge node[left] {+} (b);
        \path [->, line width=0.5mm]  (a) edge node[left] {-} (c);
        \path [->, line width=0.5mm]  (e) edge node[left] {-} (c);
        \path [->, line width=0.5mm]  (d) edge node[left] {-} (a);
        \path [->, line width=0.5mm]  (d) edge node[right] {+} (e);
    \end{tikzpicture} 
    } 
    \hspace{5pt}
    \centering
    \subfloat[$G^{\arge} \newline\phantom{ }\qquad= G_{\gets G'}(S_{\{ \arge \}})$]{
    \label{fig:example:Ge}
    \begin{tikzpicture}[scale=0.55,
        noanode/.style={scale=0.55,dashed, circle, draw=black!60, minimum size=10mm, font=\bfseries},
        unanode/.style={scale=0.55,circle, draw=black!75, minimum size=10mm, font=\bfseries},
        xunanode/.style={scale=0.55,circle, draw=black!75, minimum size=10mm, font=\bfseries, line width=2.5pt},
        invnode/.style={scale=0.55,circle, draw=white!0, minimum size=0mm, font=\bfseries},
        anode/.style={scale=0.55,circle, fill=lightgray, draw=black!60, minimum size=10mm, font=\bfseries},
        xanode/.style={scale=0.55, circle, fill=lightgray, draw=black!60, minimum size=10mm, font=\bfseries,line width=2.5pt},
        xnoanode/.style={scale=0.55, circle, dashed, draw=black!60, minimum size=10mm, font=\bfseries,line width=2.5pt}
        ]
        \node[unanode]    (a)    at(0,2)  {\argnode{\arga}{1}{1}};
        \node[unanode]  (b)    at(0,0)  {\argnode{\argb}{1}{2}};
        \node[unanode]    (c)    at(2,0)  {\argnode{\argc}{5}{1}};
        \node[xunanode]    (e)    at(3,2)  {\argnode{\arge}{3}{3}};
        \node at(5,2)  {\phantom{ }};  
        \path [->, line width=0.5mm]  (a) edge node[left] {+} (b);
        \path [->, line width=0.5mm]  (a) edge node[left] {-} (c);
        \path [->, line width=0.5mm]  (e) edge node[left] {-} (c);
    \end{tikzpicture}
    }
    \hspace{5pt}
    \centering
    \subfloat[$G^{**} \newline\phantom{ }\qquad= G_{\gets G'}(\{\arge\})$]{
    \label{fig:SSI-example:G-reversal4}\label{fig:example:G**}
    \begin{tikzpicture}[scale=0.55,
        noanode/.style={scale=0.55,dashed, circle, draw=black!60, minimum size=10mm, font=\bfseries},
        unanode/.style={scale=0.55,circle, draw=black!75, minimum size=10mm, font=\bfseries},
        xunanode/.style={scale=0.55,circle, draw=black!75, minimum size=10mm, font=\bfseries, line width=2.5pt},
        invnode/.style={scale=0.55,circle, draw=white!0, minimum size=0mm, font=\bfseries},
        anode/.style={scale=0.55,circle, fill=lightgray, draw=black!60, minimum size=10mm, font=\bfseries},
        xanode/.style={scale=0.55, circle, fill=lightgray, draw=black!60, minimum size=10mm, font=\bfseries,line width=2.5pt},
        xnoanode/.style={scale=0.55, circle, dashed, draw=black!60, minimum size=10mm, font=\bfseries,line width=2.5pt}
        ]
        \node[unanode]    (a)    at(0,2)  {\argnode{\arga}{2}{1}};
        \node[unanode]  (b)    at(0,0)  {\argnode{\argb}{1}{2}};
        \node[unanode]    (c)    at(2,0)  {\argnode{\argc}{5}{4}};
        \node[unanode]    (d)    at(1,4)  {\argnode{\argd}{1}{1}};
        \node at(5,2)  {\phantom{ }};  
        \path [->, line width=0.5mm]  (a) edge node[left] {+} (b);
        \path [->, line width=0.5mm]  (a) edge node[left] {-} (c);
        \path [->, line width=0.5mm]  (d) edge node[left] {-} (a);
    \end{tikzpicture} 
    } 
\caption{QBAFs for explanations from Example~\ref{ex:intro}/\ref{ex:inconsistency-explanations}.}
\label{fig:explanations}
\end{figure}
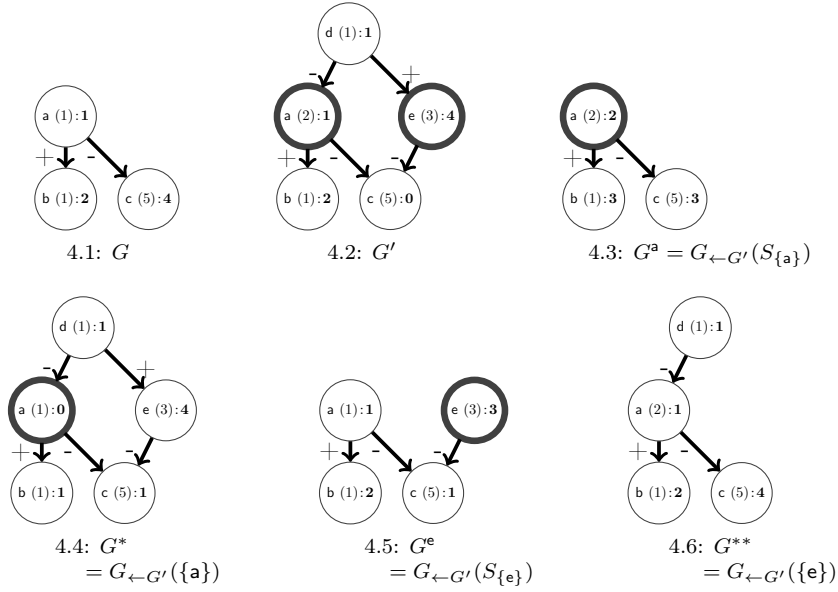
Figures~\ref{fig:example:G} and~\ref{fig:example:G'} depict again the QBAFs $G = (\{\arga, \argb, \argc\}, \is, \{(\arga, \argc)\}, \{(\arga, \argb)\})$ and 
$G' = (\{\arga, \argb, \argc, \argd, \arge\}, \is', \{(\arga, \argc), (\arge, \argc), (\argd, \arga) \}, \{(\arga, \argb), (\argd, \arge) \})$ from Example~\ref{ex:intro}, where:
\begin{itemize}
    \item $\is(\arga) = \is(\argb) = 1$ and $\is(\argc) = 5$;
    \item $\is'(\arga) = 2$, $\is'(\argb) = \is'(\argd) = 1$, $\is'(\argc) = 5$ and $\is'(\arge) = 3$.
\end{itemize}
Consider the gradual semantics $\sigma$ defined using the illustrative strength function $\fs(\argx) = \is(\argx) + \left( \sum_{\argy \in Supp(\argx)}\fs(\argy) - \sum_{\argz \in Att(\argx)}\fs(\argz) \right)$ that updates the strengths of arguments in an acyclic QBAF according to its topological ordering, as previously discussed.
Denote $\fs_{G}$ and $\fs_{G'}$ by $\fs$ and $\fs'$, respectively. 
We are primarily interested in the final strengths of the topic arguments $b$ and $c$:
$\fs(\argb) = 2 < 4 = \fs(\argc)$. 
In contrast, $\fs'(\argb) = 2 > 0 = \fs'(\argc)$. 
Hence, $\argb$ is \nscons\ w.r.t.\ $\argc$ 
($\nscon{\argb}{\argc}$),
for which we have the following sets of, respectively, (minimal) sufficient, counterfactual, and necessary explanations: 
\begin{itemize}
\item $SX_{\subset_{\min}}(\nscon{\argb}{\argc}) = \{\{\arga\}, \{\arge\}\}$;
\item $CX_{\subset_{\min}}(\nscon{\argb}{\argc}) = \{\{\arge\}\}$;
\item $NX_{\subset_{\min}}(\nscon{\argb}{\argc}) = \{\{\arga, \arge\}\}$.
\end{itemize}

Considering first $\{ \arga \}$, let us analyse the reversal $G^{\arga} \coloneqq G_{\gets G'}(S_{\{ \arga \}})$ of $G'$ to $G$ w.r.t.\ the relative complement $S_{\{ \arga \}} \coloneqq (\Args \cup \Args') \setminus \{ \arga \} = \{ \argb, \argc, \argd, \arge \}$ of $\{ \arga \}$.
$G^{\arga}$ will have the following arguments (recall Definition~\ref{def:qbaf-reversal}):
\begin{align*}
    &\big{(} \Args' \cup S_{\{ \arga \}} \big{)} \big{\backslash} \big{(} S_{\{ \arga \}} \setminus \Args \big{)} = \\
    &\left( \{ \arga, \argb, \argc, \argd, \arge \} \cup \{ \argb, \argc, \argd, \arge \} \right) \setminus \left( \{ \argb, \argc, \argd, \arge \} \setminus \{ \arga, \argb, \argc \} \right)= \{ \arga, \argb, \argc, \argd, \arge \} \setminus \{ \argd, \arge \} = \Args.
\end{align*}
To determine the initial strengths of arguments in $G^{\arga}$ (recall Definition~\ref{def:qbaf-reversal}), we find that $\Args \cap S_{\{ \arga \}} = \{ \argb, \argc \}$. 
Thus, reversing w.r.t.\ all arguments except $\arga$ yields
\begin{align*}
    &G^{\arga} = G_{\gets G'}(S_{\{ \arga \}}) = (\Args^{\arga}, \is^{\arga}, \Att^{\arga}, \Supp^{\arga}) = \\
    &\left( (\Args' \cup S_{\{ \arga \}}) \setminus (S_{\{ \arga \}} \setminus \Args), \is^{\arga}, \right. \\
    &{} \quad \left. \left(\Att'\setminus(S_{\{ \arga \}}\times\Args)\right) \cup \left((S_{\{ \arga \}}\times\Args^{\arga})\cap\Att\right) \cap \Args^{\arga} \times \Args^{\arga}, \Supp^{\arga} \right) =\\
    &\left( \Args, \{ (\arga, \is'(\arga)), (\argb, \is(\argb)), (\argc, \is(\argc))\}, \Att'\cup\emptyset, \Supp'\cup\emptyset \right) =\\
    &\left( \Args, \{ (\arga, 2), (\argb, 1), (\argc, 5) \}, \{ (\arga, \argc) \}, \{ (\arga, \argb) \} \right).
\end{align*}
So $G^{\arga}$ is like $G$ but with $\arga$'s initial strength changed to $2$, as depicted in Figure~\ref{fig:intro:Ga} and discussed in Example~\ref{ex:intro}, replicated for convenience in Figure~\ref{fig:example:Ga}.
Thus, 
$\fs_{G^{\arga}}(\argb) = 3 = \fs_{G^{\arga}}(\argc)$. 
So, $\argb$ and $\argc$ are \nscons\ (when updating from $G$ to $G'$) and remain so after reversing everything but $\{ \arga \}$ back. 
Hence, $\{ \arga \}$ is a $\subset$-minimal SSI explanation, by Definition~\ref{def:inconsistency-explanations}.

Now, let us inspect why $\{ \arga \}$ is not a CSI explanation. Observe that reversing w.r.t.\ $\arga$ yields
\begin{align*}
    &G^{*} \coloneqq G_{\gets G'}(\{ \arga \}) = (\Args^{*}, \is^{*}, \Att^*, \Supp^*)  = \\
    &\left( (\Args' \cup \{ \arga \}) \setminus (\{ \arga \} \setminus \Args), \is^{*}, \right. \\ &{} \quad 
    \left.\left(\Att'\setminus(\{ \arga \}\times\Args)\right) \cup \left((\{ \arga \}\times\Args^*)\cap\Att\right) \cap \Args^{*} \times \Args^{*}, \Supp^{*} \right) =\\
    &\left( \Args', \{ (\arga, \is(\arga)), (\argb, \is'(\argb)), (\argc, \is'(\argc)), (\argd, \is'(\argd)), (\arge, \is'(\arge)) \}, \Att', \Supp' \right) =\\
    &\left( \Args', \{ (\arga, 1), (\argb, 1), (\argc, 5), (\argd, 1), (\arge, 3) \}, \Att', \Supp' \right).
\end{align*}
So $G^{*}$ is like $G'$ but with $\arga$'s initial strength unchanged from $1$ (depicted in Figure~\ref{fig:example:G*}), thus giving  
$\fs_{G^{*}}(\argb) = 1 = \fs_{G^{*}}(\argc)$. 
That is, $\argb$ and $\argc$ do \textbf{not} become \scons\ after reversing $\{ \arga \}$ 
(i.e.\ $\nscon{\argb}{\argc}[G][G_{\gets G'}(\{ \arga \})]$), 
whence $\{ \arga \}$ is \textbf{not} a CSI explanation. 
Intuitively, even though the change to $\arga$ is minimally sufficient to bring about \nsconsy\ of $\argb$ and $\argc$ when updating from $G$ to $G'$, it is not a minimal reversal that would restore \sconsy\ back.

Now let us likewise consider $\{ \arge \}$. 
For $S_{\{ \arge \}} \coloneqq (\Args \cup \Args') \setminus \{ \arge \} = \{ \arga, \argb, \argc, \argd \}$ we find
$\left( \Args' \cup S_{\{ \arge \}} \right) \setminus \left( S_{\{ \arge \}} \setminus \Args \right) =
\{ \arga, \argb, \argc, \argd, \arge \} \setminus \left( \{ \arga, \argb, \argc, \argd \} \setminus \{ \arga, \argb, \argc \} \right) = \{ \arga, \argb, \argc, \argd, \arge \} \setminus \{ \argd \} = \{ \arga, \argb, \argc, \arge \}$.
It follows that reversing w.r.t.\ all arguments except $\arge$ yields
\begin{align*}
    &G^{\arge} \coloneqq G_{\gets G'}(S_{\{ \arge \}}) =\\
    &\left( \{ \arga, \argb, \argc, \arge \}, \{ (\arga, \is(\arga)), (\argb, \is(\argb)), (\argc, \is(\argc)), (\arge, \is'(\arge)) \}, \{ (\arga, \argc), (\arge, \argc) \}, \{ (\arga, \argb) \} \right) =\\
    &\left( \{ \arga, \argb, \argc, \arge \}, \{ (\arga, 1), (\argb, 1), (\argc, 5), (\arge, 3) \}, \{ (\arga, \argc), (\arge, \argc) \}, \{ (\arga, \argb) \} \right).
\end{align*}
So $G^{\arge}$ is $G$ with $\arge$ and the attack $(\arge, \argc)$ added (as depicted in Figure~\ref{fig:intro:Ge} and discussed in Example~\ref{ex:intro}, replicated for convenience in Figure~\ref{fig:example:Ge}), thus giving  
$\fs_{G^{\arge}}(\argb) = 2$ and $\fs_{G^{\arge}}(\argc) = 1$. 
That is, $\argb$ and $\argc$ remain \nscons\ after reversing everything but $\{ \arge \}$ back, and so $\{ \arge \}$ is a $\subset$-minimal SSI explanation.

Further, reversing w.r.t.\ $\arge$ yields
\begin{align*}
    &{} G^{**} \coloneqq G_{\gets G'}(\{ \arge \}) = (\Args^{**}, \is^{**}, \Att^{**}, \Supp^{**})=\\
    &{}( (\Args' \cup \{ \arge \}) \setminus (\{ \arge \} \setminus \Args), \is^{**}, \\ &{} \quad \left(\Att'\setminus(\{ \arge \}\times\Args)\right) \cup \left((\{ \arge \}\times\Args^{**})\cap\Att\right) \cap \Args^{**} \times \Args^{**}, \Supp^{**}) =\\
    &{}\left( \{ \arga, \argb, \argc, \argd \}, \{ (\arga, \is'(\arga)), (\argb, \is'(\argb)), (\argc, \is'(\argc)), (\argd, \is'(\argd)) \}, \{ (\arga, \argc), (\argd, \arga) \}, \{ (\arga, \argb) \} \right).
\end{align*}
$G^{**}$ is thus like $G'$ but without $\arge$ (depicted in Figure~\ref{fig:example:G**}), giving  
$\fs_{G^{**}}(\argb) = 2$ and $\fs_{G^{**}}(\argc) = 4$. 
So $\argb$ and $\argc$ \textbf{do} become \scons\ after reversing $\{ \arge \}$, 
whence $\{ \arge \}$ \textbf{is} a CSI explanation. 
Clearly, reversing w.r.t.\ $\emptyset$ yields $G'$, so that $\emptyset$ is not a CSI explanation, and hence $\{ \arge \}$ is a $\subset$-minimal CSI explanation.
Intuitively, the change to $\arge$ is both sufficient (indeed, minimally so) to bring about \nsconsy\ of $\argb$ and $\argc$ when updating from $G$ to $G'$, and it is a minimal reversal that would restore \sconsy\ back.

Now, recall that in Example~\ref{ex:intro} we already discussed that adding only $\argd$ to $G$ leaves the final strengths of $\argb$ and $\argc$ unchanged from their initial strengths, so that $\{ \argd \}$ is not an SSI explanation. It thus cannot be a CSI explanation, either. 
So, $\{ \arge \}$ is the only $\subset$-minimal CSI explanation. 

Finally, regarding NSI explanations, we have as well observed in Example~\ref{ex:intro} that while sufficient, neither $\{ \arga \}$ nor $\{ \arge \}$ is necessary: indeed, neither meets all other SSI explanations. 
However, $\{ \arga, \arge \}$ is both an SSI explanation of $\argb$ and $\argc$ (updating $G$ with only changes to $\arga$ and $\arge$ causes final strengths of $\argb$ and $\argc$ to become $3$ and $0$, respectively)
and clearly meets every other SSI explanation (because every other contains either $\arga$ or $\arge$). 
Thus, $\{\arga, \arge\} $ is a $\subset$-minimal NSI explanation.
Intuitively, a change to at least one of these two arguments is necessary to cause \nsconsy\ of $\argb$ and $\argc$ when updating from $G$ to $G'$.
Note that $\{ \arga, \argd, \arge \}$ would satisfy the first two conditions of an NSI explanation, but not the minimality one: it clearly is an SSI explanation and meets every other SSI explanation, but $\argd$ is not needed to cause the \nsconsy. In other words, necessary explanations should not contain anything that is never needed to explain \nsconsy. This illustrates why we define necessary explanations to be $\subset$-minimal, in contrast to both minimal and non-minimal versions for sufficient and counterfactual explanations.
\end{example}

Let us highlight that in our approach, all changes made to a particular argument are considered atomic.
E.g., in case several outgoing attacks are added and the argument's initial strength is changed, we cannot attribute the change in relative final strength to one particular change of the argument, but merely to the sum of changes.
We claim that in this regard, our approach trades off granularity for simplicity.
Breaking down the sum of changes to each argument would require substantial additional sophistication in the reversal definition and the overhead of this may outweigh the benefit.
As an analogy, the reader may consider explanations of causes of undesired effects to a running software system that are provided as pointers to lines of code that have been changed in the update introducing the breaking change.
Here, it is intuitively reasonable to merely focus on the line of code and use a source control system-provided visualisation of the differences between current and previous state that get a reasonably concise overview of the breaking change.

Another concern regarding argument-level atomicity are explanations generated for updates that change existing arguments instead of adding/removing them.
First, let us highlight that such changes can be considered to be less intuitive 
and not necessarily well-aligned with the philosophy of formal argumentation, where the dynamics are intuitively supposed to be modelled by (most commonly) adding or removing arguments to the graph. Here, a particularly popular concept is the notion of a normal expansion~\cite{baumann2010expanding} that models the addition of arguments to an (abstract) argumentation graph, leaving existing nodes and their relationships with each other unchanged.
Of course, limiting update operations to the addition of new arguments and to attacks/supports from and to these arguments may also make sense in quantitative bipolar argumentation.
Still, we maintain that our explanations can be useful even in scenarios where only initially existing arguments are changed.
Below, we describe and discuss an example that provides evidence for this claim.

\begin{example}
\label{ex:reviewer}
Consider $G = \QBAF = (\{\arga, \argb, \argc \}, \is, \emptyset, \emptyset)$ with $\is(\arga) = 2, \is(\argb) = 2, \is(\argc) = 1$
and two consecutive updates resulting in 
$G' = (\Args, \is, \Att', \Supp)$ where $\arga$ attacks $\argb$,
and $G'' = (\Args, \is, \Att'', \Supp)$ where $\argb$ attacks $\arga$ instead: see Figures \ref{fig:reviewer:G}, \ref{fig:reviewer:G'} and \ref{fig:reviewer:G''}. 
Let $\argb$ and $\argc$ be the topic arguments. 
Note that $\argb$ is stronger than $\argc$ to begin with.
However, $\argb$ becomes weaker when the attack from $\arga$ to $\argb$ is added (so that $\nscon{\argb}{\argc}[G][G']$) and becomes stronger again when that attack is removed and instead the attack from $\argb$ to $\arga$ is added (so that $\nscon{\argb}{\argc}[G'][G'']$). 
Let us see what the $\subset$-minimal SSI explanations are; for simplicity we consider only these explanations.

\begin{figure}[!ht]  
    \subfloat[$G$]{
    \label{fig:reviewer:G}
    \begin{tikzpicture}[scale=0.55,
        unanode/.style={scale=0.55,circle, draw=black!75, minimum size=10mm, font=\bfseries},
        ]
        \node[unanode]    (a)    at(1,0)  {\argnode{\arga}{2}{2}};
        \node[unanode]  (b)    at(0,3)  {\argnode{\argb}{2}{2}};
        \node[unanode] (c) at (2,3) {\argnode{\argc}{1}{1}};
    \end{tikzpicture}
    }
    \hspace{5pt}
    \centering
    \subfloat[$G'$]{
    \label{fig:reviewer:G'}
    \begin{tikzpicture}[scale=0.55,
        unanode/.style={scale=0.55,circle, draw=black!75, minimum size=10mm, font=\bfseries},
        ]
        \node[unanode]    (a)    at(1,0)  {\argnode{\arga}{2}{2}};
        \node[unanode]  (b)    at(0,3)  {\argnode{\argb}{2}{0}};
        \node[unanode] (c) at (2,3) {\argnode{\argc}{1}{1}};
        \path [->, line width=0.5mm]  (a) edge node[left] {-} (b);
    \end{tikzpicture} 
    }
    \hspace{5pt}
    \centering
    \subfloat[$G''$]{
    \label{fig:reviewer:G''}
    \begin{tikzpicture}[scale=0.55,
        unanode/.style={scale=0.55,circle, draw=black!75, minimum size=10mm, font=\bfseries},
        ]
        \node[unanode]    (a)    at(1,0)  {\argnode{\arga}{2}{0}};
        \node[unanode]  (b)    at(0,3)  {\argnode{\argb}{2}{2}};
        \node[unanode] (c) at (2,3) {\argnode{\argc}{1}{1}};
        \path [->, line width=0.5mm]  (b) edge node[left] {-} (a);
    \end{tikzpicture} 
    }
    \hspace{5pt}
    \centering
    \subfloat[$G^*$]{
    \label{fig:reviewer:G*}
    \begin{tikzpicture}[scale=0.55,
        unanode/.style={scale=0.55,circle, draw=black!75, minimum size=10mm, font=\bfseries},
        ]
        \node[unanode]    (a)    at(1,0)  {\argnode{\arga}{2}{\bot}};
        \node[unanode]  (b)    at(0,3)  {\argnode{\argb}{2}{\bot}};
        \node[unanode] (c) at (2,3) {\argnode{\argc}{1}{1}};
        \path [<->, line width=0.5mm]  (a) edge node[left] {-} (b);
    \end{tikzpicture} 
    }
\caption{QBAF updates affecting only relationships among arguments.}
\label{fig:reviewer}
\end{figure}
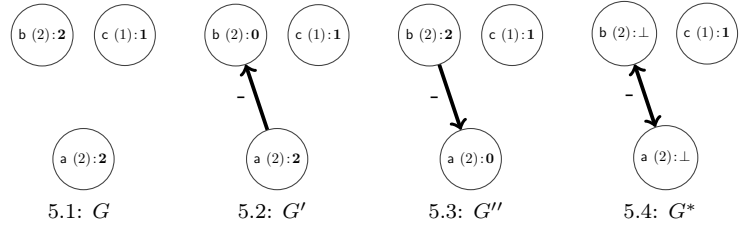

First, consider $\{ \arga \}$. The reversal of $G'$ to $G$ w.r.t.\ everything but $\arga$ yields 
$G_{\gets G'}(\Args \setminus \{\arga\}) = G_{\gets G'}(\{\argb, \argc\}) = G'$. 
We thus have $\nscon{\argb}{\argc}[G][G_{\gets G'}(\Args \setminus \{\arga\})]$, i.e.\ $\argb$ and $\argc$ remain \nscons\ after reversing everything but $\arga$ back in the first update. 
Thus, $\{ \arga \}$ is a $\subset$-minimal SSI explanation of $\argb$ and $\argc$ w.r.t.\ $G$ and $G'$: $\{ \arga \} \in SX_{\subset-min}(\nscon{\argb}{\argc}[G][G'])$. 
It intuitively allows us to indicate that the addition of the attack from $\arga$ is what explains \nsconsy.

Further, the reversal of $G''$ to $G'$ w.r.t.\ everything but $\arga$ yields 
$G'_{\gets G''}(\Args \setminus \{\arga\}) = G$. 
So $\argb$ and $\argc$ remain \nscons\ after reversing everything but $\arga$ back in the second update, too: $\nscon{\argb}{\argc}[G'][G'_{\gets G''}(\Args \setminus \{\arga\})]$. 
Hence, $\{ \arga \} \in SX_{\subset-min}(\nscon{\argb}{\argc}[G'][G''])$, which allows us to indicate that the removal of the attack from $\arga$ explains \nsconsy.

Now, we can check that $\{ \argb \}$ is not an SSI explanation of $\argb$ and $\argc$ in the first update. 
Indeed, the reversal of $G'$ to $G$ w.r.t.\ everything but $\argb$ yields $G$ itself: $G_{\gets G'}(\Args \setminus \{\argb\}) = G$. 
Thus, $\argb$ and $\argc$ become \scons: $\scon{\argb}{\argc}[G][G_{\gets G'}(\Args \setminus \{\argb\})]$ -- i.e., intuitively, the (absence of) changes to $\argb$ is clearly not sufficient to explain it becoming weaker than $\argc$ after the first update.

On the other hand, $\{ \argb \}$ is a $\subset$-minimal SSI explanation of $\argb$ and $\argc$ in the second update. 
Indeed, the reversal of $G''$ to $G'$ w.r.t.\ everything but $\argb$ yields 
$G'_{\gets G''}(\Args \setminus \{\argb\}) = G^*$ where $\arga$ and $\argb$ attack one another -- see Figure~\ref{fig:reviewer:G*}. 
There then, the final strengths of $\argb$ and $\argc$ are undefined: 
$\fs_{G^*}(\argb) = \bot = \fs_{G^*}(\argc)$. 
Consequently, $\argb$ and $\argc$ are incomparable (Definition~\ref{def:comparability}) in $G^*$: $\ncomp{\argb}{\argc}[G^*]$. 
Thus, since $\argb$ and $\argc$ are comparable in $G'$ (i.e.\ $\comp{\argb}{\argc}[G']$), we have that $\nscon{\argb}{\argc}[G'][G^*]$ (to the last bullet of Definition~\ref{def:po-consistency}) -- i.e.\ $\argb$ and $\argc$ remain \nscons\ after reversing everything but $\argb$ back in the second update. 
Therefore, $\{ \argb \} \in SX_{\subset-min}(\nscon{\argb}{\argc}[G'][G''])$. 
Intuitively, this explanation allows us to indicate that the gain of attack from $\argb$ also explains \nsconsy.
\end{example}

The previous example highlights another challenge with our strength inconsistency explanations:
considering our definition of change inconsistency, reversing to a QBAF in which one or both of the topic arguments have undefined final strengths typically retains strength inconsistency.
This can have undesirable consequences, in particular given an argumentation semantics that is only defined for acyclic QBAFs.
In the previous example, cycles leading to undefined final strengths do not cause unintuitive explanations.
However, in some cases, issues that require a dedicated fix may emerge, which we demonstrate in the example below.

\begin{example}\label{ex:cycles}
Consider the QBAFs $G = \QBAF = (\{\arga, \argb, \argc, \argd \}, \is, \{(\arga, \argd), (\argd. \argb), \\ (\argd. \argc)\}, \emptyset)$, where $\is(\arga) = \is(\argc) = \is(\argd) = 1$ and $\is(\argb) = 2$, and the update $G'$ where the attack $(\arga, \argd)$ is inverted to  $(\argd, \arga)$ and the initial strength of $\argc$ is changed to $3$ (Figures~\ref{fig:cycles:G} and \ref{fig:cycles:G'}), with topic arguments $\argb$ and $\argc$.
We have $\sigma_{G}(\argb) = 2 > \sigma_{G}(\argc) = 1$ but $\sigma_{G'}(\argb) = 2 < \sigma_{G}(\argc) = 3$, so that
$\nscon{\argb}{\argc}[G][G']$.
The intuitive explanation is that $\argc$ was changed: its initial strength has increased from $1$ to $3$.
Indeed, it holds that $\{\argc\} \in SX(\nscon{\argb}{\argc}[G][G'])$ (again, we focus on $\subset$-minimal SSI explanations).
However, it \emph{also} holds that $\{\argd\} \in SX(\nscon{\argb}{\argc}[G][G'])$: the reversal of $G'$ to $G$ w.r.t.\ everything but $\argd$ yields $G_{\gets G'}(\Args \setminus \{\argd\}) = G^*$ (Figure~\ref{fig:cycles:G*}) and $\sigma_{G^*}(\argb) = \sigma_{G^*}(\argd) = \bot$.
This is undesirable since we would not expect any argument in $\{\arga, \argd\}$ to explain a change in the relative final strengths of our topic arguments, as $\arga$ and $\argd$ affect $\argb$ and $\argc$ in the same way.
We can colloquially say that this problem is caused by the fact that while we only want to determine final strengths for acyclic QBAFs, the reversal forces us to do so for cyclic ones.
We can solve this problem in two ways:
\begin{description}
    \item[Option 1.] We can only consider reversals to a particular class of argumentation frameworks, in this case acyclic QBAFs.
    \item[Option 2.] We can adjust the definition of strength consistency so that it ignores cases in which the final strengths of topic arguments are undefined.
\end{description}
Let us claim that both adjustments can be achieved in relatively straightforward ways.
However, we consider Option 1 as more elegant (and hence use it in our implementation, see Section~\ref{implementation}): it allows us to still deal with topic arguments having undefined final strengths, in cases where this may be relevant.
Because the definitions of our strength inconsistency explanations are dependent on a class of QBAFs, they support this option.
Indeed, reversing to QBAFs that are not in the class of ``solvable'' argumentation frameworks can be generally considered confusing and counter-intuitive.

    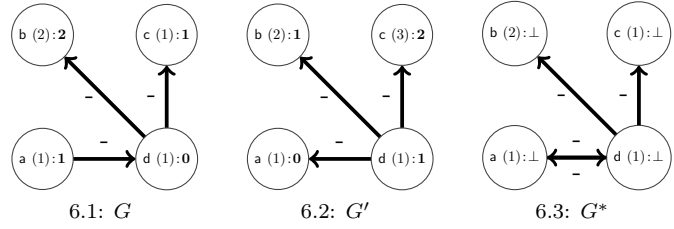
\begin{figure}[!ht]  
    \subfloat[$G$]{
    \label{fig:cycles:G}
    \begin{tikzpicture}[scale=0.55,
        unanode/.style={scale=0.55,circle, draw=black!75, minimum size=10mm, font=\bfseries},
        ]
        \node[unanode]    (a)    at(0,0)  {\argnode{\arga}{1}{1}};
        \node[unanode]    (d)    at(3,0)  {\argnode{\argd}{1}{0}};
        \node[unanode]  (b)    at(0,3)  {\argnode{\argb}{2}{2}};
        \node[unanode] (c) at (3,3) {\argnode{\argc}{1}{1}};
        \path [->, line width=0.5mm]  (a) edge node[left, above] {-} (d);
        \path [->, line width=0.5mm]  (d) edge node[left] {-} (b);
        \path [->, line width=0.5mm]  (d) edge node[left] {-} (c);
    \end{tikzpicture}
    }
    \hspace{5pt}
    \centering
    \subfloat[$G'$]{
    \label{fig:cycles:G'}
    \begin{tikzpicture}[scale=0.55,
        unanode/.style={scale=0.55,circle, draw=black!75, minimum size=10mm, font=\bfseries},
        ]
        \node[unanode]    (a)    at(0,0)  {\argnode{\arga}{1}{0}};
        \node[unanode]    (d)    at(3,0)  {\argnode{\argd}{1}{1}};
        \node[unanode]  (b)    at(0,3)  {\argnode{\argb}{2}{1}};
        \node[unanode] (c) at (3,3) {\argnode{\argc}{3}{2}};
        \path [->, line width=0.5mm]  (d) edge node[left, above] {-} (a);
        \path [->, line width=0.5mm]  (d) edge node[left] {-} (b);
        \path [->, line width=0.5mm]  (d) edge node[left] {-} (c);
    \end{tikzpicture} 
    }
    \hspace{5pt}
    \centering
    \subfloat[$G^*$]{
    \label{fig:cycles:G*}
    \begin{tikzpicture}[scale=0.55,
        unanode/.style={scale=0.55,circle, draw=black!75, minimum size=10mm, font=\bfseries},
        ]
        \node[unanode]    (a)    at(0,0)  {\argnode{\arga}{1}{\bot}};
        \node[unanode]    (d)    at(3,0)  {\argnode{\argd}{1}{\bot}};
        \node[unanode]  (b)    at(0,3)  {\argnode{\argb}{2}{\bot}};
        \node[unanode] (c) at (3,3) {\argnode{\argc}{1}{\bot}};
        \path [->, line width=0.5mm]  (d) edge node[left, above] {-} (a);
        \path [->, line width=0.5mm]  (a) edge node[left, below] {-} (d);
        \path [->, line width=0.5mm]  (d) edge node[left] {-} (b);
        \path [->, line width=0.5mm]  (d) edge node[left] {-} (c);
    \end{tikzpicture} 
    }
\caption{Cycles causing counter-intuitive explanations.}
\label{fig:cycles}
\end{figure}
\end{example}

Note that $\subset$-minimal \sconsy\ explanations are not necessarily unique: we saw this in the case of SSI explanations in Example \ref{ex:inconsistency-explanations}, but it may happen for both CSI and NSI explanations too, 
where several $\subset$-minimal explanations may be seen as relying on common ``crucial yet insufficient'' arguments.
Below, we provide an example illustrating this phenomenon.
\begin{example}\label{ex:nsi-not-unique}
Consider $G = \QBAF = (\{\arga, \argb\}, \is, \emptyset, \emptyset)$ and $G' = \QBAFF = (\{\arga, \argb, \argc, \argd, \arge\}, \is', \{(\arge, \argb)\}, \{(\argc, \arga), (\argd, \arga)\})$, both depicted in Figure~\ref{fig:necessary}, where:
\begin{itemize}
    \item $\is(\arga) = \is'(\arga) = 1$;
    \item $\is(\argb) = \is'(\argb) = 6$;
    \item $\is'(\argc) = \is'(\argd) = 2$;
    \item $\is'(\arge) = 4$.
\end{itemize}
Our topic arguments are $\arga$ and $\argb$.
We have $\fs_{G}(\arga) = 1 < \fs_{G}(\argb) = 6$, but $\fs_{G'}(\arga) = 5 > \fs_{G'}(\argb) = 2$; consequently, $\nscon{\arga}{\argb}[G][G']$.
When comparing $G$ and $G'$, we observe that the \nsconsy\ can be achieved by the minimal additions of either arguments $\{\argc, \arge\}$ (alongside support $(\argc, \arga)$ and attack $(\arge, \argb)$)  
or arguments $\{\argd, \arge\}$ (alongside support $(\argd, \arga)$ and  attack $(\arge, \argb)$); 
\emph{i.e.}, $SX_{\subset-min}(\nscon{\arga}{\argb}[G][G']) = \{\{\argc, \arge\}, \{\argd, \arge\}\}$.
Intuitively, because $\arge$ is crucial for \nsconsy\ and entailed by every sufficient \nsconsy\ explanation, it holds that $NX_{\subset-min}(\nscon{\arga}{\argb}[G][G']) = SX_{\subset-min}(\nscon{\arga}{\argb}[G][G'])$.
Colloquially, we may say that $\arge$ is a necessary argument (to change), but it is in itself not sufficient, so we cannot consider it an explanation: in addition, either $\argc$ or $\argd$ is required for a necessary explanation. 

In addition, note that even though $\arge$ is sufficient to revert in order to restore \sconsy\ (keeping both $\argc$ and $\argd$ without $\arge$ yields the final strengths of $\arga$ and $\argb$ equal to $5$ and $6$, respectively), $\{ \arge \}$ is \emph{not} a CSI explanation, precisely because it is not an SSI explanation. Instead, we have $CX_{\subset-min}(\nscon{\arga}{\argb}[G][G']) = SX_{\subset-min}(\nscon{\arga}{\argb}[G][G']) = \{\{\argc, \arge\}, \{\argd, \arge\}\}$ too. So $\subset$-minimal CSI explanations are not unique either.

\begin{figure}[!ht]  
    \subfloat[$G$]{
    \label{fig:NSI-example:G}
    \begin{tikzpicture}[scale=0.55,
        noanode/.style={scale=0.55,dashed, circle, draw=black!60, minimum size=10mm, font=\bfseries},
        unanode/.style={scale=0.55,circle, draw=black!75, minimum size=10mm, font=\bfseries},
        invnode/.style={scale=0.55,circle, draw=white!0, minimum size=0mm, font=\bfseries},
        anode/.style={scale=0.55,circle, fill=lightgray, draw=black!60, minimum size=10mm, font=\bfseries},
        ]
        \node[unanode]    (a)    at(0,0)  {\argnode{\arga}{1}{1}};
        \node[unanode]  (b)    at(3,0)  {\argnode{\argb}{6}{6}};
    \end{tikzpicture}
    }
    \hspace{10pt}
    \centering
    \subfloat[$G'$]{
    \label{fig:NSI-example:G'}
    \begin{tikzpicture}[scale=0.55,
        noanode/.style={scale=0.55,dashed, circle, draw=black!60, minimum size=10mm, font=\bfseries},
        unanode/.style={scale=0.55,circle, draw=black!75, minimum size=10mm, font=\bfseries},
        xunanode/.style={scale=0.55,circle, draw=black!75, minimum size=10mm, font=\bfseries, line width=2.5pt},
        invnode/.style={scale=0.55,circle, draw=white!0, minimum size=0mm, font=\bfseries},
        anode/.style={scale=0.55,circle, fill=lightgray, draw=black!60, minimum size=10mm, font=\bfseries},
        xanode/.style={scale=0.55, circle, fill=lightgray, draw=black!60, minimum size=10mm, font=\bfseries,line width=2.5pt},
        xnoanode/.style={scale=0.55, circle, dashed, draw=black!60, minimum size=10mm, font=\bfseries,line width=2.5pt}
        ]
        \node[unanode]    (a)    at(0,0)  {\argnode{\arga}{1}{5}};
        \node[unanode]  (b)    at(3,0)  {\argnode{\argb}{6}{2}};
        \node[xunanode]    (c)    at(-1,3)  {\argnode{\argc}{2}{2}};
        \node[xunanode]    (d)    at(1,3)  {\argnode{\argd}{2}{2}};
        \node[xunanode]    (e)    at(3,3)  {\argnode{\arge}{4}{4}};
        
        \path [->, line width=0.5mm]  (c) edge node[left] {+} (a);
        \path [->, line width=0.5mm]  (d) edge node[left] {+} (a);
        \path [->, line width=0.5mm]  (e) edge node[left] {-} (b);
    \end{tikzpicture} 
    }
\caption{$\subset$-minimal NSI and CSI explanations are not necessarily unique (Example~\ref{ex:inconsistency-explanations}).}
\label{fig:necessary}
\end{figure}
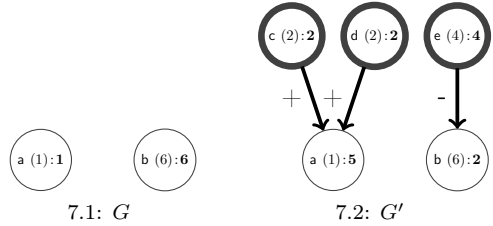
\end{example}

Before turning to formal analysis of some properties of our explanations, let us highlight that our approach can straightforwardly be adjusted to explaining the change in the final strength of a single topic argument.  For this, we merely need to introduce an unconnected ``dummy'' argument that serves as a reference point for the (changing) strength of a single topic argument. Let us illustrate this approach with another example.
 \begin{example}
 Assume we want to find out why the strength of $\argc$ decreases to equal or less than $1$, considering $G$ and $G'$ in Figure~\ref{fig:intro}/Example~\ref{ex:inconsistency-explanations}. To achieve this, we can introduce the dummy argument $\argf$ with initial strength $1$ and add it to $G$ and $G'$, without any incoming or outgoing attacks or supports (Figure~\ref{fig:intro-2}). Here, we can say that the (only) $\subset$-minimal sufficient, counterfactual, and/or necessary explanation of the change of $\argc$ to a final strength equal or less than $1$ is $\{\arge\}$. We can simply infer this from the fact that the only $\subset$-minimal sufficient/counterfactual/necessary strength inconsistency explanation of $\argc$ and $\argf$ with respect to $G$ and $G'$ is $\{\arge\}$.
 \end{example}
    \begin{figure}[!ht]
    \subfloat[$G$]{
    \label{fig:intro-2:G}
    \begin{tikzpicture}[scale=0.55,
        noanode/.style={scale=0.55,dashed, circle, draw=black!60, minimum size=10mm, font=\bfseries},
        unanode/.style={scale=0.55,circle, draw=black!75, minimum size=10mm, font=\bfseries},
        dunanode/.style={scale=0.55,circle, text=gray, draw=gray!75, minimum size=10mm, font=\bfseries},
        invnode/.style={scale=0.55,circle, draw=white!0, minimum size=0mm, font=\bfseries},
        anode/.style={scale=0.55,circle, fill=lightgray, draw=black!60, minimum size=10mm, font=\bfseries},
        ]
        \node[unanode]    (a)    at(0,2)  {\argnode{\arga}{1}{1}};
        \node[unanode]  (b)    at(0,0)  {\argnode{\argb}{1}{2}};
        \node[unanode]    (c)    at(2,0)  {\argnode{\argc}{5}{4}};
        \node[dunanode]    (f)    at(4,0)  {\argnode{\argf}{1}{1}};
        \path [->, line width=0.5mm]  (a) edge node[left] {+} (b);
        \path [->, line width=0.5mm]  (a) edge node[left] {-} (c);
    \end{tikzpicture}
    }
    \hspace{10pt}
    \centering
    \subfloat[$G'$]{
    \label{fig:intro-2:G'}
    \begin{tikzpicture}[scale=0.55,
        noanode/.style={scale=0.55,dashed, circle, draw=black!60, minimum size=10mm, font=\bfseries},
        unanode/.style={scale=0.55,circle, draw=black!75, minimum size=10mm, font=\bfseries},
        dunanode/.style={scale=0.55,circle, text=gray, draw=gray!75, minimum size=10mm, font=\bfseries},
        xunanode/.style={scale=0.55,circle, draw=black!75, minimum size=10mm, font=\bfseries, line width=2.5pt},
        invnode/.style={scale=0.55,circle, draw=white!0, minimum size=0mm, font=\bfseries},
        anode/.style={scale=0.55,circle, fill=lightgray, draw=black!60, minimum size=10mm, font=\bfseries},
        xanode/.style={scale=0.55, circle, fill=lightgray, draw=black!60, minimum size=10mm, font=\bfseries,line width=2.5pt},
        xnoanode/.style={scale=0.55, circle, dashed, draw=black!60, minimum size=10mm, font=\bfseries,line width=2.5pt}
        ]
        \node[unanode]    (a)    at(0,2)  {\argnode{\arga}{2}{1}};
        \node[unanode]  (b)    at(0,0)  {\argnode{\argb}{1}{2}};
        \node[unanode]    (c)    at(2,0)  {\argnode{\argc}{5}{0}};
        \node[unanode]    (d)    at(1,4)  {\argnode{\argd}{1}{1}};
        \node[xunanode]    (e)    at(3,2)  {\argnode{\arge}{3}{4}};
        \node[dunanode]    (f)    at(4,0)  {\argnode{\argf}{1}{1}};
        
        \path [->, line width=0.5mm]  (a) edge node[left] {+} (b);
        \path [->, line width=0.5mm]  (a) edge node[left] {-} (c);
        \path [->, line width=0.5mm]  (e) edge node[left] {-} (c);
        \path [->, line width=0.5mm]  (d) edge node[left] {-} (a);
        \path [->, line width=0.5mm]  (d) edge node[right] {+} (e);
    \end{tikzpicture}
    } 
\caption{Explaining the change of the final strength of the single topic argument $\argc$ when updating from $G$ to $G'$ from Figure~\ref{fig:intro}, with the introduction of a dummy topic argument $\argf$ (grey) to $G$ and $G'$.}
\label{fig:intro-2}
\end{figure}
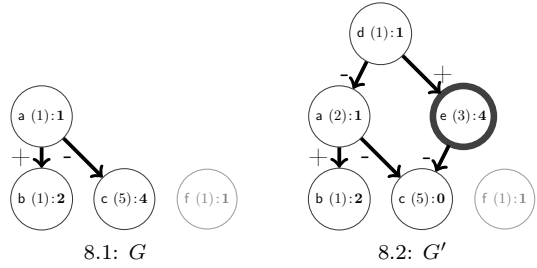

\section{Theoretical Analysis}
\label{analysis}

In this section, we let $G = \QBAF$ and $G' = \QBAFF$ be QBAFs, $\argx, \argy \in \Args \cap \Args'$, and $\fs$ be a strength function. 
We show that both minimal sufficient, counterfactual, and necessary explanations are sound and complete: either we have \nsconsy\ and at least one non-empty set (and no empty set) of explanation arguments, or we have \sconsy\ explained by (and only by) the empty set.

First, if two arguments are \scons, then there is no \nsconsy\ to explain and the only explanation is the empty set ($SX$-soundness). 

\begin{proposition}[$SX$-Soundness]\label{prop:ssi-a}
If $\scon{\argx}{\argy}$, then $SX(\nscon{\argx}{\argy}) = \{\emptyset\}$.
\end{proposition}
\begin{proof}
Let $\scon{\argx}{\argy}$.
Then $\emptyset$ is an SSI explanation directly by Definition~\ref{def:inconsistency-explanations}.
On the other hand, no $S \neq \emptyset$ can be an SSI explanation, by definition, precisely because $\scon{\argx}{\argy}$. So $SX(\nscon{\argx}{\argy}) \subseteq \{ \emptyset \}$. 
Hence, $SX(\nscon{\argx}{\argy}) = \{\emptyset\}$ as required. 
\end{proof}

If arguments are \nscons\ though, then there exists an explanation, but no empty explanation ($SX$-completeness). 

\begin{proposition}[$SX$-Completeness]\label{prop:ssi-b}
If $\nscon{\argx}{\argy}$, then $|SX(\nscon{\argx}{\argy})| \geq 1$ and $\emptyset \not \in SX(\nscon{\argx}{\argy})$.
\end{proposition}
\begin{proof}
Let $\nscon{\argx}{\argy}[G][G']$.

\noindent \textbf{Proof of $|SX(\nscon{\argx}{\argy}[G][G'])| \geq 1$.} 

\noindent
By definition of an SSI explanation, since $\nscon{\argx}{\argy}[G][G']$, any $S \subseteq \Args \cup \Args'$ is an SSI explanation of $\argx$ and $\argy$ (w.r.t.\ $\fs$, $G$, and $G'$) iff $\nscon{\argx}{\argy}[G][G_{\leftarrow G'}((\Args \cup \Args')\setminus S)]$. 
Suppose for a contradiction that such a set $S$ does not exist: $\forall S \subseteq \Args \cup \Args'$, $\scon{\argx}{\argy}[G][G_{\leftarrow G'}((\Args \cup \Args')\setminus S)]$. 
In particular, for $\Args \cup \Args'$, we have $\scon{\argx}{\argy}[G][G_{\leftarrow G'}((\Args \cup \Args')\setminus (\Args \cup \Args'))]$. 
Since $G_{\leftarrow G'}((\Args \cup \Args')\setminus(\Args \cup \Args')) = G_{\leftarrow G'}(\emptyset) = G'$ by definition of QBAF reversal (Definition~\ref{def:qbaf-reversal}), it follows that $\scon{\argx}{\argy}[G][G']$, contradicting $\nscon{\argx}{\argy}[G][G']$. 
By contradiction, there is at least one $S \in SX(\nscon{\argx}{\argy}[G][G'])$.

\noindent \textbf{Proof of $\emptyset \not \in SX(\nscon{\argx}{\argy}[G][G'])$.}

\noindent
Suppose for a contradiction $\emptyset \in SX(\nscon{\argx}{\argy}[G][G'])$. 
Since $\nscon{\argx}{\argy}[G][G']$, we have $\nscon{\argx}{\argy}[G][G_{\leftarrow G'}((\Args \cup \Args')\setminus \emptyset)]$, by definition of an SSI explanation. 
As $G_{\leftarrow G'}(\Args \cup \Args') = G$ by definition of QBAF reversal, it follows that $\nscon{\argx}{\argy}[G][G]$. 
But this is in direct contradiction to the definition of \sconsy\ (Definition~\ref{def:po-consistency}). 
Thus, $\emptyset \not\in SX(\nscon{\argx}{\argy}[G][G'])$. 
\end{proof}
Trivially, the same soundness and completeness properties apply to $\subset$-minimal SSI explanations as well.
\begin{corollary}[$SX_{\subset_{\min}}$-Soundness]\label{prop:ssi-c}
    If $\scon{\argx}{\argy}$, then $SX_{\subset_{\min}}(\nscon{\argx}{\argy}) = \{\emptyset\}$.
\end{corollary}
\begin{proof}
If $\scon{\argx}{\argy}$, $\emptyset$ is clearly $\subset$-minimal, and hence the proof follows directly from Proposition~\ref{prop:ssi-a}.
\end{proof}
\begin{corollary}[$SX_{\subset_{\min}}$-Completeness]\label{prop:ssi-d} 
If $\nscon{\argx}{\argy}$, then $|SX_{\subset_{\min}}(\nscon{\argx}{\argy})| \geq 1$ and $\emptyset \not \in SX_{\subset_{\min}}(\nscon{\argx}{\argy})$.
\end{corollary}
\begin{proof}
   The proof follows directly from Proposition~\ref{prop:ssi-b}: if $\nscon{\argx}{\argy}$, there exists at least one $\subset$-minimal set in $SX(\nscon{\argx}{\argy})$ and hence $SX_{\subset_{\min}}(\nscon{\argx}{\argy})$ is not empty, either; and if $\emptyset \not \in SX(\nscon{\argx}{\argy})$, then the $\subset$-minimal sets in $SX(\nscon{\argx}{\argy})$ cannot be empty, either.
\end{proof}

We can also prove analogous properties for ($\subset$-minimal) CSI explanations.
\begin{proposition}[$CX$-Soundness]\label{prop:csi-a}
If $\scon{\argx}{\argy}$, then $CX(\nscon{\argx}{\argy}) = \{\emptyset\}$.
\end{proposition}
\begin{proof}
Let $\scon{\argx}{\argy}[G][G']$. 
By definition, a CSI explanation is an SSI explanation $S$ for which $\scon{\argx}{\argy}[G][G_{\leftarrow G'}(S)]$. 
Since $G_{\leftarrow G'}(\emptyset) = G'$ and $\scon{\argx}{\argy}[G][G']$, we find $\scon{\argx}{\argy}[G][G_{\leftarrow G'}(\emptyset)]$, whence $\emptyset$ is a CSI explanation. Because $\emptyset$ is the unique SSI explanation (Proposition~\ref{prop:ssi-a}), it must also be the unique CSI explanation.
\end{proof}

\begin{proposition}[$CX$-Completeness]
\label{prop:csi-b}
If $\nscon{\argx}{\argy}$, then $|CX(\nscon{\argx}{\argy})| \geq 1$ and $\emptyset \not\in CX(\nscon{\argx}{\argy})$.
\end{proposition}
\begin{proof}
Let $\nscon{\argx}{\argy}[G][G']$. 

\noindent \textbf{Proof of $|CX_{\subset_{\min}}(\nscon{\argx}{\argy}[G][G'])| \geq 1$.} 

\noindent
Consider $C = \Args \cup \Args'$. 
First, as in the proof of Proposition~\ref{prop:ssi-b}, $C \in SX(\nscon{\argx}{\argy}[G][G'])$. 
Since $G_{\leftarrow G'}(C) = G$ and it holds by definition that $\scon{\argx}{\argy}[G][G]$, we conclude with $\scon{\argx}{\argy}[G][G_{\leftarrow G'}(C)]$.
Thus, $C \in CX(\nscon{\argx}{\argy}[G][G'])$, so that $CX(\nscon{\argx}{\argy}[G][G'])$ is non-empty and $|CX(\nscon{\argx}{\argy}[G][G'])| \geq 1$.

\noindent \textbf{Proof of $\emptyset \not \in CX(\nscon{\argx}{\argy}[G][G'])$.} 

\noindent
Since a CSI explanation is an SSI explanation, if $\emptyset$ were an CSI explanation, then $\emptyset$ would be an SSI explanation, contradicting Proposition~\ref{prop:ssi-b}.
\end{proof}

\begin{proposition}[$CX_{\subset_{\min}}$-Soundness]\label{prop:csi-c}
If $\scon{\argx}{\argy}$, then $CX_{\subset_{\min}}(\nscon{\argx}{\argy}) = \{\emptyset\}$.
\end{proposition}
\begin{proof}
If $\scon{\argx}{\argy}$, $\emptyset$ is clearly $\subset$-minimal, and hence the proof follows directly from Proposition~\ref{prop:csi-a}.
\end{proof}
\begin{proposition}[$CX_{\subset_{\min}}$-Completeness]\label{prop:csi-d}
If $\nscon{\argx}{\argy}$, then $|CX_{\subset_{\min}}(\nscon{\argx}{\argy})| \geq 1$ and $\emptyset \not\in CX_{\subset_{\min}}(\nscon{\argx}{\argy})$.
\end{proposition}
\begin{proof}
The proof follows directly from Proposition~\ref{prop:csi-b}: if $\nscon{\argx}{\argy}$, there exists at least one $\subset$-minimal set in $CX(\nscon{\argx}{\argy})$ and hence $CX_{\subset_{\min}}(\nscon{\argx}{\argy})$ is not empty, either; and if $\emptyset \not \in CX(\nscon{\argx}{\argy})$, then the $\subset$-minimal sets in $CX(\nscon{\argx}{\argy})$ cannot be empty either.
\end{proof}

Analogous properties for ($\subset$-minimal) NSI explanations follow from the soundness and completeness results obtained for SSI explanations.
\begin{proposition}[$NX_{\subset_{\min}}$-Soundness]\label{prop:nsi-a} 
If $\scon{\argx}{\argy}$, then $NX_{\subset_{\min}}(\nscon{\argx}{\argy}) = \{\emptyset\}$.
\end{proposition}
\begin{proof}
    Proposition~\ref{prop:ssi-a} has established that if $\scon{\argx}{\argy}$ then $SX(\nscon{\argx}{\argy}) = \{\emptyset\}$. 
    Since $\emptyset$ is a $\subset$-minimal subset of any set, 
    we have by definition of NSI explanations (Definition~\ref{def:inconsistency-explanations}), that if $\scon{\argx}{\argy}$, then $NX_{\subset_{\min}}(\nscon{\argx}{\argy}) = \{\emptyset\}$.
\end{proof}
\begin{proposition}[$NX_{\subset_{\min}}$-Completeness]\label{prop:nsi-b}
If $\nscon{\argx}{\argy}$, then $|NX_{\subset_{\min}}(\nscon{\argx}{\argy})| \geq 1$ and $\emptyset \not \in NX_{\subset_{\min}}(\nscon{\argx}{\argy})$.
\end{proposition}
\begin{proof}
    Proposition~\ref{prop:ssi-b} has established that if $\nscon{\argx}{\argy}$ then $|SX(\nscon{\argx}{\argy})| \geq 1$ and $\emptyset \not \in SX(\nscon{\argx}{\argy})$. 
    Note that as in the proof of Proposition~\ref{prop:ssi-b}, $N = \Args \cup \Args' \in SX(\nscon{\argx}{\argy})$ as long as $\nscon{\argx}{\argy}$, because $G_{\leftarrow G'}((\Args \cup \Args')\setminus N) = G_{\leftarrow G'}(\emptyset) = G'$. 
    Then observe that since $\emptyset \not \in SX(\nscon{\argx}{\argy})$, it holds that
    $\nexists S \subseteq (\Args \cup \Args') \setminus N = \emptyset$ such that $S \in SX(\nscon{\argx}{\argy}[G][G'])$. 
    Thus, since $\subseteq$ is a well-ordering on $\Args \cup \Args'$, either $N$ itself or some proper subset of $N$ is a $\subset$-minimal set satisfying the two other conditions from the definition of NSI explanations (Definition~\ref{def:inconsistency-explanations}). 
    Consequently, $|NX_{\subset_{\min}}(\nscon{\argx}{\argy})| \geq 1$. 
    Clearly, $\emptyset \not \in NX_{\subset_{\min}}(\nscon{\argx}{\argy})$, because $\emptyset$ is not an SSI explanation.
\end{proof}

The above results show that there are non-trivial (i.e.\ non-empty) sufficient, counterfactual, and necessary \nsconsy\ explanations if and only if a \nsconsy\ results between two arguments after an update to a given QBAF. 
We deem this a desirable property: one needs to explain only if a change in the relative strengths of arguments actually happens after an update; and if there are explanations of changes in the relative strengths of arguments, then the explanations should correctly refer to such changes.

\section{Searching for Minimal Strength Inconsistency Explanations}
\label{implementation}
%
In the previous section we defined and analysed \sconsy\ explanations in a generic manner.
In order to implement the search/generation of explanations, it is useful to make additional assumptions to then advance the analysis with the objective of speeding up the search.
Below, we present such a formal analysis, focusing on $\subset$-minimal explanations given aggregation-influence semantics (Definition~\ref{definition:aggregation-influence}) for \emph{acyclic} QBAFs, and subsequently informally describe further implementation aspects (Subsection~\ref{sec:implementation-analysis}).
Then, we briefly describe the software implementation, alongside a basic empirical evaluation (Subsection~\ref{sec:empirical}).

\subsection{Analysis and Algorithms}
\label{sec:implementation-analysis}
A crucial aspect of determining minimal \sconsy\ explanations is to assess which arguments can potentially be part of the explanations without having to apply reversals based on all the elements in the powerset of all the arguments in the two QBAFs.
For this, we conduct an analysis as to which arguments can potentially be part of a minimal explanation -- which we call \emph{potential explanations} -- and which cannot.

Here, we assume a focus on acyclic QBAFs and aggregation-influence semantics, following Option 1 as discussed in Example~\ref{ex:cycles} and discarding reversals to cyclic QBAFs, thus bypassing the question of \emph{convergence} (or lack thereof) of aggregation-influence semantics in face of cycles.
We reasonably assume that convergence for the acyclic QBAFs we are interested in implementation-wise is guaranteed (see~\cite{Potyka:2019}).
The following properties aid the search for minimal SSI and CSI explanations.
\begin{enumerate}
    \item Given acyclic QBAFs, we base our explanations always on reversals to acyclic QBAFs to avoid issues like the one illustrated by Example~\ref{ex:cycles}. We show that this does not affect explanation soundness and completeness (Corollaries~\ref{cor:acylic-ssi-a}~and~\ref{cor:acylic-ssi-b}.
    \item The directional connectedness principle (Principle~\ref{principle:directional-connectedness} below), which is satisfied by all aggregation-influence semantics, implies that, when searching for explanations, we only need to consider arguments that can reach the topic arguments (Corollary~\ref{corollary:aggregation-influence-directionality} and Proposition~\ref{prop:reachability}).
    \item The search space can be reduced by excluding every argument that has the same initial strength and the same outgoing attacks and supports in both the original and the updated QBAF (Proposition~\ref{prop:unchanged-exclusion}).
    \item Utilizing the minimality property, the search through the powerset of the remaining arguments can start with the empty set (even assuming \sconsy\ has not yet been established) and continue with $\subset$-minimal sets among the remaining sets of arguments that have not yet been searched. This ensures the search terminates fast, as strictly larger sets (with respect to set inclusion) do not need to be searched as soon as one or several explanations have been found (Corollary~\ref{corollary:set-inclusion}).
\end{enumerate}

The following properties help narrow down potential explanations when searching for (minimal) NSI explanations.
\begin{enumerate}
    \item An NSI explanation is always also an SSI explanation (Proposition~\ref{prop:NSI-meet}).
    \item An NSI explanation is a subset of or equal to the union of all SSI explanations (Proposition~\ref{prop:nsi-minimal-ssi}).
    \item For every minimal SSI explanation, an NSI explanation must contain at least one argument of the SSI explanation (Corollary~\ref{corollary:nsi-contains-ssi}). 
\end{enumerate}
Let us first show that given acyclic QBAFs and topic arguments therein, (minimal) strength inconsistency explanations are sound and complete even if we make them dependent on reversals to acyclic QBAFs only.
Here, we denote the class of acyclic QBAFs by ${\cal G}_{ac}$.
Reversals of an updated to an initial QBAF using the empty set result in the unchanged updated QBAF, which is, given our assumption, obviously acyclic. Hence, for explanation soundness the limitation to acyclic reversals is intuitively irrelevant.
\begin{lemma}
[Acyclic (Minimal) $SX$-, $CX$-, and $NX$-Soundness]\label{cor:acylic-ssi-a}
If $G, G' \in {\cal G}_{ac}$ and $\scon{\argx}{\argy}[G][G']$, then for \\ $X \in \{SX_{{\cal G}_{ac}}, SX_{\subset-min, {\cal G}_{ac}}, CX_{{\cal G}_{ac}}, CX_{\subset-min, {\cal G}_{ac}}, NX_{\subset-min, {\cal G}_{ac}}\}$ it holds that $X(\nscon{\argx}{\argy}[G][G']) = \{\emptyset\}$.
\end{lemma}
\begin{proof}
For $X = SX_{{\cal G}_{ac}}$, the proof follows directly from Definition~\ref{def:inconsistency-explanations} (if $\scon{\argx}{\argy}[G][G']$ then $SX_{{\cal G}_{ac}}(\nscon{\argx}{\argy}[G][G']) = \{\emptyset\})$, analogously to Proposition~\ref{prop:ssi-a}.
For $X = SX_{\subset-min, {\cal G}_{ac}}$, the proof follows as a consequence ($\emptyset$ must then be the only minimal SSI), analogously to Corollary~\ref{prop:ssi-c}.
For $X = CX_{{\cal G}_{ac}}$, the proof follows from the case $X = SX_{{\cal G}_{ac}}$, analogously to how Proposition~\ref{prop:csi-a} follows from Proposition~\ref{prop:ssi-a}: because $\emptyset$ is an NSI explanation and $G_{\leftarrow G'}(\emptyset) = G'$ and $\scon{\argx}{\argy}[G][G']$, we find $\scon{\argx}{\argy}[G][G_{\leftarrow G'}(\emptyset)]$, whence $\emptyset$ is a CSI explanation. Because $\emptyset$ is the only NSI explanation it is also the only CSI explanation.
For $X = CX_{\subset-min, {\cal G}_{ac}}$, the proof follows from the previous case, analogously to Proposition~\ref{prop:csi-c}.
Finally, for $X = NX_{\subset-min, {\cal G}_{ac}}$, the proof follows from the case $X = SX_{{\cal G}_{ac}}$, analogously to how Proposition~\ref{prop:nsi-a} follows from Proposition~\ref{prop:ssi-a}: $SX_{{\cal G}_{ac}}(\nscon{\argx}{\argy}[G][G']) = \{\emptyset\})$ and $\emptyset$ is a $\subset$-minimal subset of any set; hence $NX_{\subset-min,{\cal G}_{ac}}(\nscon{\argx}{\argy}[G][G']) = \{\emptyset\})$.
\end{proof}
Analogously, we can provide completeness results given acyclic QBAFs: in case of strength inconsistency, there is always the initial, and hence an acyclic, QBAF to revert to.
\begin{lemma}
[Acyclic (Minimal) $SX$-, $CX$-, and $NX$-Completeness]\label{cor:acylic-ssi-b}
If $G, G' \in {\cal G}_{ac}$ and $\nscon{\argx}{\argy}[G][G']$, then for \\ $X \in \{SX_{{\cal G}_{ac}}, SX_{\subset-min, {\cal G}_{ac}}, CX_{{\cal G}_{ac}}, CX_{\subset-min, {\cal G}_{ac}}, NX_{\subset-min, {\cal G}_{ac}}\}$ it holds that $|X(\nscon{\argx}{\argy}[G][G'])| \geq 1$ and $\emptyset \not \in X(\nscon{\argx}{\argy}[G][G'])$.
\end{lemma}
\begin{proof}
For $X = SX_{{\cal G}_{ac}}$, the proof follows directly from Definition~\ref{def:inconsistency-explanations}: analogously to Proposition~\ref{prop:ssi-b}, because $G_{\leftarrow G'}(\emptyset) = G'$ and $G'$ is acyclic, by definition of QBAF reversal (Definition~\ref{def:qbaf-reversal}), there must exist a set $S \subseteq Args \cup Args'$ s.t. $S \in SX_{{\cal G}_{ac}}(\nscon{\argx}{\argy}[G][G'])$; because $G_{\leftarrow G'}(\Args \cup \Args') = G$, it must hold that $\emptyset \not \in SX_{{\cal G}_{ac}}(\nscon{\argx}{\argy}[G][G'])$.
For $X = SX_{\subset-min, {\cal G}_{ac}}$, the proof follows as a consequence (a non-empty set must then be the only minimal SSI), analogously to Corollary~\ref{prop:ssi-d}.
For $X = CX_{{\cal G}_{ac}}$, the proof follows from the case $X = SX_{{\cal G}_{ac}}$ analogously to how Proposition~\ref{prop:csi-b} follows from Proposition~\ref{prop:ssi-b}: $G_{\leftarrow G'}(C) = G$ and, $G$ is acyclic, and it holds by definition that $\scon{\argx}{\argy}[G][G]$.
Then, the proof for $X = CX_{\subset-min, {\cal G}_{ac}}$ follows from the previous case, analogously to Proposition~\ref{prop:csi-d}.
Finally, for $X = NX_{\subset-min, {\cal G}_{ac}}$, the proof follows from Proposition~\ref{prop:ssi-b} analogously to the way the proof of Proposition~\ref{prop:nsi-b} follows from Proposition~\ref{prop:ssi-b}: in case of strength inconsistency, we have $S = Args \cup Args' \in SX_{{\cal G}_{ac}}(\nscon{\argx}{\argy}[G][G'])$ and $\emptyset \not \in SX_{{\cal G}_{ac}}(\nscon{\argx}{\argy}[G][G'])$ and hence, by definition of $NX$, there must exist $S' \subseteq Args \cup Args'$ s.t. $S \neq \emptyset$ and $S' \in NX_{{\cal G}_{ac}}(\nscon{\argx}{\argy}[G][G'])$. 
\end{proof}
Given these results, we can avoid reversals to cyclic QBAFs, which may lead to counter-intuitive behaviour (see Example~\ref{ex:cycles}).
Here and henceforth, we assume that if a semantics yields undefined final strengths for cyclic QBAFs, strength inconsistency explanations are determined with respect to the class of acyclic QBAFs.
Implementation-wise, we can achieve this by discarding cyclic QBAFs once we encounter cycles when reversing to a QBAF.

Let us now define the directional connectedness principle and show that it is satisfied by every aggregation-influence semantics.
Recall from \nameref{prelim} that for $\arga, \argb \in \Args$, $\argb$ is reachable from $\arga$ (in $G$), denoted $r_{G}(\arga, \argb)$, iff there is a path from $\arga$ to $\argb$ in the directed graph $(\Args, \Att \cup \Supp)$.
\begin{principle}[Directional Connectedness~\cite{tree}]
\label{principle:directional-connectedness}
A gradual semantics $\fs$ satisfies the \emph{directional connectedness principle} iff for every QBAF $G = \QBAF$, for any $\argx \in \Args$ it holds that 
$\fs_{G}(\argx) = \fs_{G \downarrow_{\{ \argx \} \cup \{\argy |\argy \in \Args, r_{G}(\argy, \argx)\}}}(\argx)$.
\end{principle}

This principle says that an argument can only affect the final strength of another argument if the latter argument is reachable from the former. 
This principle is obviously satisfied by aggregation-influence semantics (Definition~\ref{definition:aggregation-influence}), e.g.\ by the semantics presented in Table~\ref{table:semanticsExamples}.
\begin{corollary}
\label{corollary:aggregation-influence-directionality}
Every aggregation-influence semantics satisfies the directional connectedness principle.
\end{corollary}
\begin{proof}
This follows directly from Definition~\ref{definition:aggregation-influence} of aggregation-influence semantics, for which the final strength of any argument depends only on (the strengths of) its parents.
\end{proof}

Subsequently then, an argument should appear in a minimal explanation of \nsconsy\ only if it affects the strength of any of the topic arguments. 
To formalise this, let us define $r_{G,G'}(\arga, \argx, \argy) := r_{G}(\arga, \argx) \lor r_{G}(\arga, \argy) \lor r_{G'}(\arga, \argx) \lor r_{G'}(\arga, \argy)$ as a short-hand (given $\arga, \argx, \argy \in \Args \cup \Args'$), which holds iff either of $\argx$ and $\argy$ is reachable from $\arga$ in either of $G$ and $G'$. 
The next result shows that $\subset$-minimal explanations do not contain arguments from which topic arguments are not reachable.

\begin{proposition}\label{prop:reachability}
    Given an aggregation-influence semantics and $\arga \in \Args \cup \Args'$, if $\neg r_{G,G'}(\arga, \argx, \argy)$, then $\arga \not \in \bigcup_{S \in SX_{\subset_{\min}}(\nscon{\argx}{\argy}) \cup NX_{\subset_{\min}}(\nscon{\argx}{\argy})  \cup CX_{\subset_{\min}}(\nscon{\argx}{\argy})} S$.
\end{proposition}
\begin{proof}
    Because every aggregation-influence semantics satisfies the directional connectedness principle (Corollary~\ref{corollary:aggregation-influence-directionality}), for every $S' \subseteq \{\argb | \argb \in \Args \cup \Args', \neg r_{G,G'}(\argb, \argx, \argy) \}$ 
    it holds that $\fs_{G'}(\argx) = \fs_{G'\downarrow_{S'}}(\argx)$ and $\fs_{G'}(\argy) = \fs_{G'\downarrow_{S'}}(\argy)$.
    It follows from the definitions of $SX$, $NX$, and $CX$ (Definition~\ref{def:inconsistency-explanations}) that for $S \subseteq \Args \cup \Args'$, $X \in \{SX, NX_{\subset_{\min}}, CX\}$, if $S \in X(\nscon{\argx}{\argy})$ then $S \setminus S' \in X(\nscon{\argx}{\argy})$.
    Hence, because of the $\subset$-minimality condition in $SX_{\subset_{\min}}$, $NX_{\subset_{\min}}$, and $CX_{\subset_{\min}}$, it follows that $\arga \not \in \bigcup_{S \in SX_{\subset_{\min}}(\nscon{\argx}{\argy}) \cup NX_{\subset_{\min}}(\nscon{\argx}{\argy})  \cup CX_{\subset_{\min}}(\nscon{\argx}{\argy})} S$, as required. 
\end{proof}
In addition to the (un)reachability aspect, let us show that arguments whose initial strengths and outgoing attacks/supports are the same in both the original and the updated framework are never part of any minimal explanation.
\begin{proposition}\label{prop:unchanged-exclusion}
    Let $\arga \in \Args \cap \Args'$ be such that $(\{\arga\} \times \Args) \cap \Att  = (\{\arga\} \times \Args') \cap \Att'$, $(\{\arga\} \times \Args) \cap \Supp  = (\{\arga\} \times \Args') \cap \Supp'$, and $\is(\arga) = \is'(\arga)$.
    Then $\arga \not \in \bigcup_{S \in SX_{\subset_{\min}}(\nscon{\argx}{\argy}) \cup NX_{\subset_{\min}}(\nscon{\argx}{\argy})  \cup CX_{\subset_{\min}}(\nscon{\argx}{\argy})} S$.
\end{proposition}
\begin{proof}
It follows from the definition of QBAF reversal (Definition~\ref{def:qbaf-reversal}) that for every $S \subseteq \Args \cup \Args'$ it holds that $G_{\leftarrow G'}(S) = G_{\leftarrow G'}(S \setminus \{\arga\})$.
    If it were the case that, for $X \in \{SX_{\subset_{\min}}, NX_{\subset_{\min}}, CX_{\subset_{\min}}\}$, there was $S \in X(\nscon{\argx}{\argy})$ with $\arga \in S$, then we would have $S' = S \setminus \{\arga\}$ such that $S' \in X(\nscon{\argx}{\argy})$. 
    But this would contradict $\subset$-minimality of $S$. 
    Hence, there is no $S \in X(\nscon{\argx}{\argy})$ with $\arga \in S$. 
    Thus, $\arga \not \in \bigcup_{S \in SX_{\subset_{\min}}(\nscon{\argx}{\argy}) \cup NX_{\subset_{\min}}(\nscon{\argx}{\argy})  \cup CX_{\subset_{\min}}(\nscon{\argx}{\argy})} S$.
\end{proof}
Finally, let us observe that obviously, we can find all $\subset$-minimal explanations by following the order imposed by the set inclusion relation on the powerset of arguments (starting with set inclusion-smaller sets first) and terminate the search as soon as the explanations we have found are properly contained in all the items in the powerset that have not yet been searched.
\begin{corollary}\label{corollary:set-inclusion}
Let $S \subset S' \subseteq \Args \cup \Args'$ and $X \in \{SX_{\subset_{\min}}, NX_{\subset_{\min}}, CX_{\subset_{\min}}\}$. If $S \in X(\nscon{\argx}{\argy})$, then $S' \not \in X(\nscon{\argx}{\argy})$.
\end{corollary}
\begin{proof}
    This follows directly from the $\subset$-minimality conditions of $SX_{\subset_{\min}}$, $NX_{\subset_{\min}}$, and $CX_{\subset_{\min}}$ (Definition~\ref{def:inconsistency-explanations}).
\end{proof}

Based on the analysis, we can design a simple algorithm for identifying minimal sufficient and counterfactual \nsconsy\ explanations of topic arguments $\argx$ and $\argy$, according to the following steps (note that we use algorithm-style notation and in particularly $\leftarrow$ for variable assignments and updates; also see Algorithm~\ref{alg:minSCSI}).
Let us note again that we assume an aggregation-influence semantics and acyclic QBAFs.
\begin{enumerate}
    \item Check if \nsconsy\ exists; if not, return $\{\emptyset\}$:\\
\texttt{if $\scon{\argx}{\argy}$: return $\{\emptyset\}$}.
    \item Assign the empty set to the variable $exps$ that will eventually contain all explanations:\\
    $exps \leftarrow \emptyset$.
    \item Generate a sorted (increasing, by set cardinality) list of explanations based on the set of \emph{potential explanations}: the powerset (excluding the empty set) of the set containing all arguments of both QBAFs, with the exceptions of arguments that cannot reach any of the topic arguments as well as arguments whose strengths and outgoing attackers and supporters have not been changed: \\
    $\textit{potential\_exps} \leftarrow sort(2^{\{(\Args \cup \Args') \setminus (\Args_{\neg r} \cup \Args_{*})\}} \setminus \emptyset)$, where:
    \begin{itemize}
        \item $\Args_{\neg r} := \{\arga | \arga \in \Args \cup \Args', \neg r_{G,G'}(\arga, \argx, \argy)\}$;
        \item $\Args_{*} := \{\arga | \arga \in \Args \cap \Args', (\{\arga\} \times \Args) \cap \Att  = (\{\arga\} \times \Args') \cap \Att', (\{\arga\} \times \Args) \cap \Supp  = (\{\arga\} \times \Args') \cap \Supp', \is(\arga) = \is'(\arga)\}$;
        \item $sort$ is a function that sorts a set of sets into a totally ordered list according to set cardinality in increasing order.
    \end{itemize}
    \item Run a \emph{while} loop over all potential explanations (terminating when all potential explanations have been explored): \\
    \texttt{While $potential\_exps \neq \emptyset$}:
        \begin{enumerate}[i)]
            \item Assign the current potential explanation from the end of the list whilst removing the element: \\
            $\textit{p\_exp} \leftarrow \textit{potential\_exps.pop()}$.
            \item If the potential explanation is not a superset of an already identified explanation, check if the potential explanation is actually an explanation, and if so, add it to the set of explanations: \\
            \texttt{if not $\exists exp \in exps: exp \subset \textit{p\_exp}$ \textbf{AND} $p\_exp$ is an $X$ explanation}, where $X \in \{SSI, CSI\}$ (depending on the type of explanation that is to be generated): $exps \leftarrow exps \cup \{p\_exp\}$.
        \end{enumerate}    
    \item Finally, return the determined explanations: \\
    \texttt{return $exps$.}
\end{enumerate}

Figure~\ref{fig:explanations-search} illustrates the search for $\subset$-minimal SSI explanations for our running example.
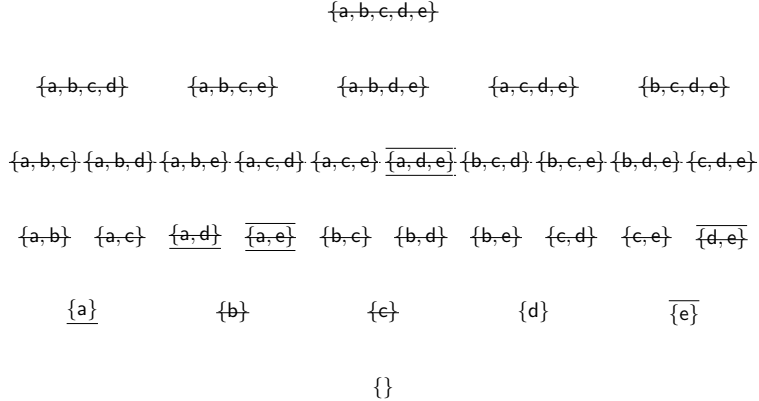
\begin{figure}[!ht]
        \centering
    \begin{tikzpicture}[scale=1,
        basic/.style={scale=0.8, draw=black!0, minimum size=10mm, font=\bfseries},
        red/.style={scale=0.8, minimum size=10mm, draw=black!0, font=\bfseries
        },
        blue/.style={scale=0.8, minimum size=10mm, draw=black!0, font=\bfseries
        },
        violet/.style={scale=0.8, minimum size=10mm, draw=black!0, font=\bfseries,
        },
        ]
        \node[basic]    (empty)    at(4.5,0)  {$\{\}$};
        \node[red]  (a)    at(0.5,1)  {\underline{$\{\arga\}$}};
        \node[basic]  (b)    at(2.5,1)  {\st{$\{\argb\}$}};
        \node[basic]  (c)    at(4.5,1)  {\st{$\{\argc\}$}};
        \node[basic]  (d)    at(6.5,1)  {$\{\argd\}$};
        \node[blue]  (e)    at(8.5,1)  {$\overline{\{\arge\}}$};
        \node[basic]  (ab)    at(0,2)  {\st{$\{\arga, \argb\}$}};
        \node[basic]  (ac)    at(1,2)  {\st{$\{\arga, \argc\}$}};
        \node[red]  (ad)    at(2,2)  {\underline{\st{$\{\arga, \argd\}$}}};
        \node[violet]  (ae)    at(3,2)  {\underline{\st{$\overline{\{\arga, \arge\}}$}}};
        \node[basic]  (bc)    at(4,2)  {\st{$\{\argb, \argc\}$}};
        \node[basic]  (bd)    at(5,2)  {\st{$\{\argb, \argd\}$}};
        \node[basic]  (be)    at(6,2)  {\st{$\{\argb, \arge\}$}};
        \node[basic]  (cd)    at(7,2)  {\st{$\{\argc, \argd\}$}};
        \node[basic]  (ce)    at(8,2)  {\st{$\{\argc, \arge\}$}};
        \node[blue]  (de)    at(9,2)  {\st{$\overline{\{\argd, \arge\}}$}};
        \node[basic]  (abc)    at(0,3)  {\st{$\{\arga, \argb, \argc\}$}};
        \node[basic]  (abd)    at(1,3)  {\st{$\{\arga, \argb, \argd\}$}};
        \node[basic]  (abe)    at(2,3)  {\st{$\{\arga, \argb, \arge\}$}};
        \node[basic]  (acd)    at(3,3)  {\st{$\{\arga, \argc, \argd\}$}};
        \node[basic]  (ace)    at(4,3)  {\st{$\{\arga, \argc, \arge\}$}};
        \node[violet]  (ade)    at(5,3)  {\underline{\st{$\overline{\{\arga, \argd, \arge\}}$}}};
        \node[basic]  (bcd)    at(6,3)  {\st{$\{\argb, \argc, \argd\}$}};
        \node[basic]  (bce)    at(7,3)  {\st{$\{\argb, \argc, \arge\}$}};
        \node[basic]  (bde)    at(8,3)  {\st{$\{\argb, \argd, \arge\}$}};
        \node[basic]  (cde)    at(9,3)  {\st{$\{\argc, \argd, \arge\}$}};
        \node[basic]  (abcd)    at(0.5,4)  {\st{$\{\arga, \argb, \argc, \argd\}$}};
        \node[basic]  (abce)    at(2.5,4)  {\st{$\{\arga, \argb, \argc, \arge\}$}};
        \node[basic]  (abde)    at(4.5,4)  {\st{$\{\arga, \argb, \argd, \arge\}$}};
        \node[basic]  (acde)    at(6.5,4)  {\st{$\{\arga, \argc, \argd, \arge\}$}};
        \node[basic]  (bcde)    at(8.5,4)  {\st{$\{\argb, \argc, \argd, \arge\}$}};
        \node[basic]  (abcde)    at(4.5,5)  {\st{$\{\arga, \argb, \argc, \argd, \arge\}$}};

    \end{tikzpicture} 
\caption{Searching for $\subset$-minimal SSI explanations in Example~\ref{ex:intro}. The figure lists all sets in $2^{\Args \cup \Args'}$. Our SSI explanations $\{\arga\}$ and $\{\arge\}$ are highlighted in red and blue, respectively. All sets that contain $\argb$ or $\argc$ can be excluded, because neither the initial strengths of these arguments nor their outgoing attacks or supports have been changed (highlighted by strike-through). All other sets that can be excluded from the search because of only $\{\arga\}$ / $\{\arge\}$ are underlined/overlined (with strike-through).
The remaining sets (as well as our explanation sets) have to be evaluated with respect to their ability to explain the \nsconsy.}
\label{fig:explanations-search}
\end{figure}
As we can see in Steps 3 and 4, the search space\footnote{Here, we use the term \emph{search space} to refer to the number of sets of arguments for which we need to check whether they are explanations.} for explanations is reduced from $|2^{\Args \cup \Args'}|$ to $|2^{(\Args \cup \Args') \setminus (\Args_{\neg r} \cup \Args_{*})} \setminus \Args_{**}|$, where $\Args_{**} := \{S_{*} | S_{*} \subseteq \Args \cup \Args', \exists S \in X(\nscon{\argx}{\argy})$ s.t.\ $S \subset S_{*}\}$ for $X \in \{ SX, CX \}$. 
Note that technically, the empty set is excluded as well, as long as it is clear that we have an actual \nsconsy.
The algorithm can be applied for determining all sufficient and counterfactual \nsconsy\ explanations.

\begin{algorithm}[ht!]
\caption{Find all $\subset$-minimal SSI or CSI explanations}\label{alg:minSCSI}
\footnotesize{
\begin{algorithmic}[1]
\Require QBAFs $G = \QBAF$, $G' = \QBAFF$, arguments $\argx, \argy \in \Args \cap \Args'$, X $\in \{$SSI, CSI$\}$
\Ensure $exps \subseteq 2^{\Args \cup \Args'}$
    \Function{minimalSXCX}{$G, G', \argx, \argy$, X}
\If {$\scon{\argx}{\argy}$}
\State \textbf{return } $\{\emptyset\}$
\EndIf
\State $exps \leftarrow \emptyset$
\State $\textit{potential\_exps} \leftarrow sort(2^{\{(\Args \cup \Args') \setminus (\Args_{\neg r} \cup \Args_{*})\}} \setminus \emptyset)$
\While {$potential\_exps \neq \emptyset$}
\State $\textit{p\_exp} \leftarrow \textit{potential\_exps.pop()}$
\If {not $\exists exp \in exps: exp \subset \textit{p\_exp}$ \textbf{AND} $p\_exp$ is an X explanation}
\State $exps \leftarrow exps \cup \{p\_exp\}$
\EndIf
\EndWhile
\State \textbf{return } $\textit{exps}$
\EndFunction
\end{algorithmic}
}
\end{algorithm}
Let us highlight that Lines 6-9 of Algorithm~\ref{alg:minSCSI} resemble the algorithm used for frequent item set computation in association rule mining~\cite{DBLP:conf/vldb/AgrawalS94}.

For finding NSI explanations we need to extend the algorithm above. 
Again, we can guide our search by observing some formal properties.
In particular, we know already how to find minimal SSI explanations and we want to utilise this during our search for NSI explanations.

First, we can show that in case of \nsconsy, every NSI explanation meets every minimal SSI explanation.
\begin{proposition}%
\label{prop:NSI-meet}
 If $\nscon{\argx}{\argy}$ then $\forall N \in NX_{\subset_{\min}}(\nscon{\argx}{\argy})$, $\forall S \in SX_{\subset_{\min}}(\nscon{\argx}{\argy})$ we have $N \cap S \neq \emptyset$.
\end{proposition}
\begin{proof}
    Let $\nscon{\argx}{\argy}$, $N \in NX_{\subset_{\min}}(\nscon{\argx}{\argy})$ and $S \in SX_{\subset_{\min}}(\nscon{\argx}{\argy})$. Because $\nscon{\argx}{\argy}$, it follows from Corollaries~\ref{prop:ssi-c}~and~\ref{prop:nsi-a} that $N \neq \emptyset$ and $S \neq \emptyset$.
    Now, assume for a contradiction that $N \cap S = \emptyset$.
    This establishes that $S \subseteq (\Args \cup \Args') \setminus N$, which together with $S$ being an SSI explanation contradicts $N$ being an NSI explanation (see Definition~\ref{def:inconsistency-explanations}). 
    By contradiction, $N \cap S \neq \emptyset$, and since $S$ and $N$ were arbitrary, this proves the proposition. 
\end{proof}

Also, we know that an NSI explanation does not contain any arguments that are not contained in some minimal SSI explanation.
\begin{proposition}\label{prop:nsi-minimal-ssi}
   $\forall N \in NX_{\subset_{\min}}(\nscon{\argx}{\argy})$ it holds that $N \subseteq \bigcup_{S \in SX_{\subset_{\min}}(\nscon{\argx}{\argy})} S$.
\end{proposition}
\begin{proof}
    First consider the case $\scon{\argx}{\argy}$. By definition of SSI and NSI explanations (Definition~\ref{def:inconsistency-explanations}), 
    $\emptyset$ is a unique SSI and hence NSI explanation, and it is a subset of any set, which proves the proposition for this case.
    For the case $\nscon{\argx}{\argy}$, let us define $S^* \coloneqq \bigcup_{S \in SX_{\subset_{\min}}(\nscon{\argx}{\argy})} S$. Suppose for a contradiction that for some $\emptyset \neq N \in NX_{\subset_{\min}}(\nscon{\argx}{\argy})$, it holds that $N \not \subseteq S^*$, i.e.\ that at least some element from $N$ is not in any $\subset$-minimal SSI explanation. 
    We will show this leads to a contradiction to $\subset$-minimality of $N$ as an NSI explanation. 

    Recall that an NSI explanation is an SSI explanation $N'$ for which it holds that
    $\nexists S \subseteq (\Args \cup \Args') \setminus N \text{ s.t.\ } S \in SX(\nscon{\argx}{\argy}[G][G'])$,
    and $N'$ is $\subset$-minimal such set. 
    Now, let us remove $N\setminus S^*$ from $N$: 
    define $N^* := N\setminus(N\setminus S^*) = N \cap S^*$. 
    Since $N$ meets every $S \in SX_{\subset_{\min}}$ (by Proposition~\ref{prop:NSI-meet}),  
    and since $(N \setminus N^*) \cap S^* = \emptyset$,
    we know that $N^*$ (still) meets every $S \in SX_{\subset_{\min}}$. 
    Hence, it holds that
    $\nexists S \subseteq (\Args \cup \Args') \setminus N^* \text{ s.t.\ } S \in SX(\nscon{\argx}{\argy}[G][G'])$.
    Now, as $N^* \subset N$, to contradict the minimality of $N$ as an NSI explanation, we merely need to show that $N^*$ is an SSI explanation itself. 
    
    Suppose for a contradiction that $N^*$ is not an SSI explanation. 
    Then again, as $(N \setminus N^*) \cap S^* = \emptyset$, and since $S^*$ is the union of all $\subset$-minimal SSI explanation, this would mean $N$ cannot be an SSI, either (because adding arguments that are not in any minimal SSI explanation never turns $N^*$ into an SSI explanation, since it would continue to fail to be a superset of any minimal SSI explanation). 
    So this contradicts $\subset$-minimality of $N$ as an NSI explanation, as required. 
    Hence, there is no non-empty $N \in NX_{\subset_{\min}}(\nscon{\argx}{\argy})$ with $N \not \subseteq S^*$. 
    It thus follows that $\forall N \in NX_{\subset_{\min}}(\nscon{\argx}{\argy})$ it holds that $N \subseteq \bigcup_{S \in SX_{\subset_{\min}}(\nscon{\argx}{\argy})} S$.
\end{proof}
Finally, we observe that for every minimal SSI explanation, an NSI explanation must contain at least one argument of the SSI explanation.
\begin{corollary}\label{corollary:nsi-contains-ssi}
   $\forall N \in NX_{\subset_{\min}}(\nscon{\argx}{\argy})$ it holds that $\forall S \in SX_{\subset_{\min}}(\nscon{\argx}{\argy})$, $\exists \arga \in S$ s.t. $\arga \in N$.
\end{corollary}
\begin{proof}
    The proof follows directly from the definitions of $\subset$-min SSI and NSI explanations (Definition~\ref{def:inconsistency-explanations}): suppose for a contradiction that $\exists N \in NX_{\subset_{\min}}(\nscon{\argx}{\argy})$ such that $\exists S \in SX_{\subset_{\min}}(\nscon{\argx}{\argy})$ and $\nexists \arga \in S$ s.t. $\arga \in N$: then $N$ and $S$ are disjoint, which violates the Definition of a $\subset$-min NSI, thus establishing the contradiction.
\end{proof}

These results can be utilised when implementing an algorithm that, given two QBAFs and two topic arguments that occur in both QBAFs, explain \nsconsy\ (or \sconsy).
The algorithm consists of the following steps (see Algorithm~\ref{alg:minNCI}).
\begin{enumerate}
    \item Check if \nsconsy\ exists; if no, return $\{\emptyset\}$:\\
    \texttt{if $\scon{\argx}{\argy}$: return $\{\emptyset\}$}.
    \item Determine the $\subset$-minimal SSI explanations using Algorithm~\ref{alg:minSCSI}: \\
    $\textit{mssis} \gets \textsc{minimalSXCX}(G, G', \argx, \argy, \textnormal{SSI})$.
    \item Generate a sorted (increasing, by set cardinality) list of explanations based on the set of \emph{potential explanations}: the powerset of the union of all $\subset$-minimal SSI explanations, but from this powerset all sets that have an empty intersection with at least one $\subset$-minimal SSI explanation are excluded: \\
    $\textit{potential\_exps} \gets sort(2^{\bigcup_{S \in mssis} S} \setminus \{S' | S' \in 2^{\bigcup_{S \in mssis} S}, \exists S'' \in mssis, S' \cap S'' = \emptyset\})$.
    \item Assign the empty set to the variable $exps$ that will eventually contain all explanations:\\
    $exps \leftarrow \emptyset$.
    \item Run a \emph{while} loop over all potential explanations (terminating when all potential explanations have been explored): \\
    \texttt{While $potential\_exps \neq \emptyset$}:
        \begin{enumerate}[i)]
            \item Assign the current potential explanation from the end of the list whilst removing the element: \\
            $\textit{p\_exp} \leftarrow \textit{potential\_exps.pop()}$.
            \item If the potential explanation is not a superset of an already identified explanation, check if the potential explanation is actually an explanation, and if so, add it to the set of explanations\footnote{The (non-)existence checks are carried out with simple for-loops. Also, let us observe that $\nexists mssi \in mssis: mssi \subseteq (\Args \cup \Args') \setminus p\_exp$ iff $\nexists S' \subseteq (\Args \cup \Args') \setminus p\_exp: S'$ is an SSI explanation; hence, we can rely on the previously determined set of $\subseteq$-min SSI explanations.}: \\
            \texttt{if $p\_exp$ is an SSI explanation \textbf{AND} not $\exists exp \in exps: exp \subset \textit{p\_exp}$ AND not $\exists mssi \in mssis: mssi \subseteq (\Args \cup \Args') \setminus p\_exp$: $exps \leftarrow exps \cup \{p\_exp\}$}. 
        \end{enumerate}
    \item Finally, return the determined explanations: \\
    \texttt{return $exps$.}
\end{enumerate}

\begin{algorithm}[ht!]
\caption{Find all $\subset$-minimal NSI explanations}\label{alg:minNCI}
\footnotesize{
\begin{algorithmic}[1]
\Require QBAFs $G = \QBAF$, $G' = \QBAFF$, arguments $\argx, \argy \in \Args \cap \Args'$
\Ensure $exps \subseteq 2^{\Args \cup \Args'}$
    \Function{minimalNX}{$G, G', \argx, \argy$}
\If {$\scon{\argx}{\argy}$}
\State \textbf{return } $\{\emptyset\}$
\EndIf
\State $\textit{mssis} \gets \textsc{minimalSXCX}(G, G', \argx, \argy, \textnormal{SSI})$
\State $\textit{potential\_exps} \gets sort(2^{\bigcup_{S \in mssis} S} \setminus \{S' | S' \in 2^{\bigcup_{S \in mssis} S}, \exists S'' \in mssis, S' \cap S'' = \emptyset\})$
\State $\textit{exps} \gets \emptyset$
\While {$potential\_exps \neq \emptyset$}
\State $\textit{p\_exp} \leftarrow \textit{potential\_exps.pop()}$
\If {$p\_exp_{SX [\nscon{\argx}{\argy}]}$ \textbf{AND} \\
\quad \quad \quad \quad \textbf{not} $\exists exp \in exps: exp \subset \textit{p\_exp}$ \textbf{AND} \\ 
\quad \quad \quad \quad \textbf{not} $\exists mssi: mssis: mssi \subseteq (\Args \cup \Args') \setminus S$}
\State $exps \leftarrow exps \cup \{p\_exp\}$
\EndIf
\EndWhile
\State \textbf{return } $\textit{exps}$
\EndFunction
\end{algorithmic}
}
\end{algorithm}

\subsection{Implementation}
\label{sec:empirical}
The algorithms for determining strength inconsistency explanations have been implemented, alongside a tool that allows for the modelling of and inference from acyclic QBAFs using the semantics from Table~\ref{table:semanticsExamples}, as well as the basic semantics that we use for demonstrating purposes.
The tool is implemented in C with Python bindings, combining a performant programming language with one that is easy to use.
The source code, alongside documentation, tests, and examples, is available at \url{http://s.cs.umu.se/t6xfz2}.

\begin{figure}[!ht]
    \centering
    \subfloat[Average final strength computation time (in seconds, y-axis) per number of arguments (x-axis).\label{subfig:eval-fs}]{
        \includegraphics[width=0.46\textwidth]{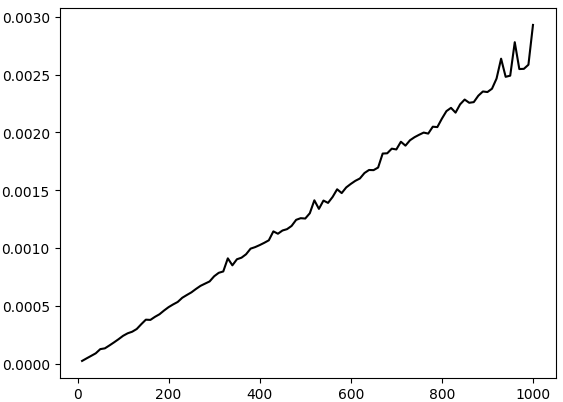}
    }
    \hspace{5pt}
    \subfloat[Average explanation computation time (in seconds, y-axis) per number of arguments (x-axis).\label{subfig:eval-avg}]{
        \includegraphics[width=0.46\textwidth]{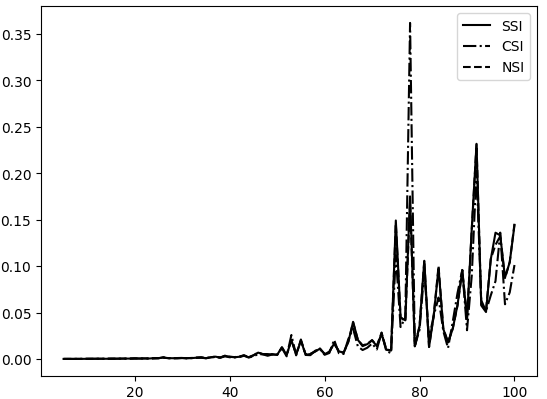}
    }
\caption{A basic empirical evaluation provides an intuition of the implementation's capabilities and limitations. Final strength computation is somewhat fast and scales well; computing all SSI, CSI, or NSI explanations can take substantial time for frameworks that approach 100 arguments.}
\label{fig:eval}
\end{figure}

In order to provide an intuition of how the implementation performs, we conducted a basic empirical evaluation.
The experiments were run using a Google Colab setup supported by 25.5GB RAM and four 2.20 GHz Intel Xeon CPUs with 56320 KB cache; the applied semantics was quadratic energy semantics.
As a starting point, we computed and plotted the average time for determining all final strengths of QBAFs with 5 to 1000 arguments and an average of 3 randomly generated outgoing edges per argument (Figure~\ref{subfig:eval-avg}). Per QBAF size (in number of arguments), we generated 50 QBAFs.
The results indicate that computing final strengths is reasonably fast and increases approximately linearly, even if the number edges per argument are somewhat high and the QBAFs are somewhat large.
Then, we computed and plotted the average time for determining all SSI, CSI and NSI explanations (per explanation type) for QBAFs with 5 to 100 arguments, an average of 1.1 of initial outgoing edges per argument\footnote{Note that then, the generated QBAFs have more edges than a tree.} and update operations that amount to approximately 20\% of the nodes (Figure~\ref{subfig:eval-avg}). The changes are roughly evenly distributed among the five types \emph{strength update}, \emph{edge removal}, \emph{edge addition}, \emph{argument removal}, and \emph{argument addition}, the latter with on average more outgoing than incoming edges per added argument. Again, per QBAF size (in number of arguments), we generated 50 QBAFs.
The findings show that explanation generation is reasonably fast for QBAFs of smaller sizes that are not particularly densely connected; i.e., the implemented tool can presumably handle argumentation graphs that model the statements and relationships in ``human-like'' argumentation dialogues reasonably well. During the experiment, less than 0.1\% of explanations had cyclic reversals as potential explanations\footnote{While we can guarantee that acyclic explanations exist (see Subsection~\ref{sec:implementation-analysis}), during the search for potential explanations, we may encounter cyclic QBAFs, which we then disregard. Our results indicate that at least for the (relatively sparse) QBAFs we have generated for our empirical evaluation, this happens rather rarely.}. This indicates that our handling of cyclic reversals given an acyclic initial QBAF and its acyclic update has little impact on the explanations that are generated, at least for QBAFs that are not particularly densely connected.
However, the findings also indicate that drawing inferences from substantially denser or larger QBAFs is costly (and, based on our experiences, occupies substantial amounts of working memory).
The maximal time needed for computing explanations was around 12.5 seconds, for the CSI explanations that are reflected by the spike in the graph at around 78 arguments.
Clearly, in order to service ambitious envisioned use cases that require large and dense argumentation frameworks, substantial improvements are required, at least if all explanation sets are to be computed.
For example, one could approximate minimal explanations by sampling potential explanations, similarly to how sampling is applied in the context of Shapley values~\cite{CASTRO20091726}, which are often used in explainable AI.
Other performance improvements may be achievable by relatively straightforward engineering-oriented changes of the ``research code'' used for the evaluation: \emph{e.g.}, the way the set of potential explanations is generated and maintained can presumably be substantially improved in terms of memory efficiency.
We consider these directions, as well as the comparison of our basic QBAF solver with existing benchmarks~\cite{DBLP:conf/kr/Potyka18}, interesting future work.

\section{Discussion}
\label{discussion}
Before concluding the paper, let us illustrate how our explanation approach can be applied to abstract argumentation, the arguably most foundational formal argumentation approach (Subsection~\ref{applicability}) and position our work in the context of related research (Subsection~\ref{related}).
\subsection{Applicability to Abstract Argumentation}
\label{applicability}
Our approach to generate change explanations is novel\footnote{See below for a discussion of related research.} even when considering other graph-based formal argumentation approaches and can be straightforwardly applied to at least some of them. The arguably best-known formal argumentation approach is abstract argumentation~\cite{dung1995acceptability}.
In abstract argumentation, Argumentation Frameworks (AFs) are directed graphs (without weighted nodes) and the only relation is the attack relation.
Based on the nodes and their attack relation, sets of nodes that can be considered jointly acceptable -- so-called \emph{extensions} -- are inferred by abstract argumentation \emph{semantics}. For example, given the graph $F = (\{\arga, \argb, \argc\}, \{(\argc, \argb), (\argb, \arga)\})$, the only extension typical \emph{semantics}\footnote{See~\cite{baroni2018abstract} for a survey of abstract argumentation semantics.} infer is $\{\arga, \argc\}$: although $\argb$ attacks $\arga$, $\argb$ is, in turn, attacked by the unattacked argument $\argc$, which thus reinstates $\arga$.
In order to allow for more expressivity, abstract argumentation has been extended in several ways, e.g.\ to support an additional \emph{preference relation}~\cite{Amgoud2002}, or to resolve conflicts between logical statements~\cite{Dung:Kowalski:Toni:2006}.
In order to apply our change explanation approach to abstract argumentation, we merely need to map AFs to QBAFs and their inference statuses to final strengths (note that this means we disregard the notion of an extension).
For this, we proceed as follows.
\begin{itemize}
    \item We define $\mathbb{I} := \{n, c, s \}$ and stipulate that $\preceq$ is the transitive and reflexive closure of $\{(n, c), (c, s)\}$.
    \item We map the AF $(AR, AT)$ to the QBAF $G = \{AR, \{(\arga, \bot) | \arga \in AR \}, AT, \{\}\}$.
    \item Given an abstract argumentation semantics, we determine the final strength $\sigma(\arga)$ of an argument $\arga \in AR$ as follows:
    \begin{enumerate*}[i)]
        \item $\sigma(\arga) = s$ if $\arga$ is in the intersection of all extensions that our abstract argumentation semantics yields.
        \item $\sigma(\arga) = c$ else if $\arga$ is in the union of all extensions that our abstract argumentation semantics yields.
        \item $\sigma(\arga) = n$, otherwise.
    \end{enumerate*}
    Then, the final strength $s$ corresponds to what is typically called \emph{sceptical acceptance} in abstract argumentation, whereas $c$ corresponds to \emph{credulous acceptance}; we can say that a final strength of $n$ means clear rejection.
\end{itemize}
With this mapping, we can then explain relative changes in acceptance statuses of arguments in abstract argumentation frameworks.
Let us illustrate this with an example.
Here, we make use of stable semantics as our abstract argumentation semantics.
Given an AF $(AR, AT)$, $S \subseteq AR$ is a stable extension iff for no $\arga, \argb \in S$, $(\arga, \argb) \in AT$ holds (conflict-freeness) and for every $\argc \in AR \setminus S$ it holds that $\exists \argd \in S$ such that $(\argd, \argc) \in AT$ holds.
Consider the AFs $(\{\arga, \argb, \argc\}, \{(\arga, \argb), (\argb, \arga), (\argc, \argb)\})$ and $(\{\arga, \argb, \argc, \argd, \arge, \argf\}, \{(\arga, \argb), (\argb, \arga), (\argc, \argb), (\argc, \argf), (\argd, \argc), (\argd, \argf), (\arge, \arga), (\argf, \argb)\})$ that we can then map to the QBAFs $F$ and $F'$ (respectively) in Figure~\ref{fig:abstract-arg}.
Our topic arguments are $\arga$ and $\argb$.
In $F$, $\arga$'s final strength is $s$, i.e.\ the argument is sceptically accepted in the corresponding AF, whereas the final strength of $\argb$ is $n$ (clear rejection).
In $F'$, we have a final strength of $n$ for $\arga$ and of $s$ for $\argb$.
Our $\subset$-minimal strength inconsistency explanations are as follows:
\begin{itemize}
    \item $SX_{\subset_{\min}}(\nscon{\arga}{\argb}) = \{\{\argd\}, \{\arge\}\}$;
    \item $CX_{\subset_{\min}}(\nscon{\arga}{\argb}) = NX_{\subset_{\min}}(\nscon{\arga}{\argb}) = \{\{\argd, \arge\}\}$.
\end{itemize}
Intuitively, either $\argd$ or $\arge$ are sufficient for changing the relative final strength of $\arga$ and $\argb$, leading to final strength equality: credulous acceptance and clear rejection, respectively.
However, the relative strength of the topic argument remains intact only without changes to $\argd$ \emph{and} $\arge$ (counterfactual explanation) and one of these two arguments is necessary to affect the change (necessary explanation).
Yet, the reader may have observed that from a practical reasoning perspective, the sufficient explanations may not be convincing: if we apply changes only to $\argd$ or $\arge$, both $\arga$ and $\argb$ are either credulously accepted (changes to $\argd$) or clearly rejected (changes to $\arge$) and we may prefer a ``stronger'' explanation.
However, this issue can be easily addressed by introducing a weaker notion of strength consistency.
We can tweak Definition~\ref{def:po-consistency} and stipulate that a change from stronger to weaker or vice versa is required for a lack of strength consistency, i.e.\ that a change to equality is insufficient.
Let us claim that such a change makes more sense in the abstract argumentation setting, where we have discrete acceptance statuses instead of continuous final strengths.

Future work could explore explanations in the change of the inferred extensions, going beyond the acceptance statuses of individual arguments. For example, we may want to know why we infer extension $E$ and no longer $E'$, with $E' \not \supseteq E$ and hence to compute some sort of sufficient, necessary, or counterfactual explanations for this. Here, our previous work on explaining the violation of monotonicity of entailment in the abstract argumentation semantics introduced in Dung's seminal paper may serve as an additional starting point~\cite{clarxai}.

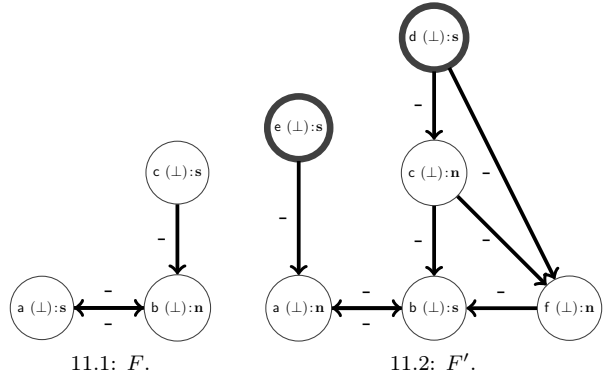
\begin{figure}[!ht]
    \subfloat[$F
    $.\label{subfig:f}]{
        \begin{tikzpicture}[scale=0.6,
            noanode/.style={scale=0.55,dashed, circle, draw=black!60, minimum size=10mm, font=\bfseries},
            unanode/.style={scale=0.55,circle, draw=black!75, minimum size=10mm, font=\bfseries},
            xunanode/.style={scale=0.55,circle, draw=black!75, minimum size=10mm, font=\bfseries, line width=2.5pt},
            invnode/.style={scale=0.55,circle, draw=white!0, minimum size=0mm, font=\bfseries},
            anode/.style={scale=0.55,circle, fill=lightgray, draw=black!60, minimum size=10mm, font=\bfseries},
            xanode/.style={scale=0.55, circle, fill=lightgray, draw=black!60, minimum size=10mm, font=\bfseries,line width=2.5pt},
            xnoanode/.style={scale=0.55, circle, dashed, draw=black!60, minimum size=10mm, font=\bfseries,line width=2.5pt}
            ]
            \node[unanode]      (a)    at(0,0)  {\argnode{\arga}{\bot}{s}};
            \node[unanode]      (b)    at(3,0)  {\argnode{\argb}{\bot}{n}};
            \node[unanode]      (c)    at(3,3)  {\argnode{\argc}{\bot}{s}};
            \path [->, line width=0.5mm]  (a) edge node[left, above] {-} (b);
            \path [->, line width=0.5mm]  (b) edge node[left, below] {-} (a);
            \path [->, line width=0.5mm]  (c) edge node[left] {-} (b);
        \end{tikzpicture}
    }
    \hspace{10pt}
    \centering
    \subfloat[$F'
    $.\label{subfig:fprime}]{
        \begin{tikzpicture}[scale=0.6,
            noanode/.style={scale=0.55,dashed, circle, draw=black!60, minimum size=10mm, font=\bfseries},
            unanode/.style={scale=0.55,circle, draw=black!75, minimum size=10mm, font=\bfseries},
            xunanode/.style={scale=0.55,circle, draw=black!75, minimum size=10mm, font=\bfseries, line width=2.5pt},
            invnode/.style={scale=0.55,circle, draw=white!0, minimum size=0mm, font=\bfseries},
            anode/.style={scale=0.55,circle, fill=lightgray, draw=black!60, minimum size=10mm, font=\bfseries},
            xanode/.style={scale=0.55, circle, fill=lightgray, draw=black!60, minimum size=10mm, font=\bfseries,line width=2.5pt},
            xnoanode/.style={scale=0.55, circle, dashed, draw=black!60, minimum size=10mm, font=\bfseries,line width=2.5pt}
            ]
            \node[unanode]      (a)    at(0,0)  {\argnode{\arga}{\bot}{n}};
            \node[unanode]      (b)    at(3,0)  {\argnode{\argb}{\bot}{s}};
            \node[unanode]      (c)    at(3,3)  {\argnode{\argc}{\bot}{n}};
            \node[unanode]      (f)    at(6,0)  {\argnode{\argf}{\bot}{n}};
            \node[xunanode]      (d)    at(3,6)  {\argnode{\argd}{\bot}{s}};
            \node[xunanode]      (e)    at(0,4)  {\argnode{\arge}{\bot}{s}};
            \path [->, line width=0.5mm]  (a) edge node[left, above] {-} (b);
            \path [->, line width=0.5mm]  (b) edge node[left, below] {-} (a);
            \path [->, line width=0.5mm]  (c) edge node[left] {-} (b);
            \path [->, line width=0.5mm]  (d) edge node[left] {-} (c);
            \path [->, line width=0.5mm]  (e) edge node[left] {-} (a);
            \path [->, line width=0.5mm]  (d) edge node[left] {-} (f);
            \path [->, line width=0.5mm]  (c) edge node[left] {-} (f);
            \path [->, line width=0.5mm]  (f) edge node[left, above] {-} (b);
        \end{tikzpicture}
    }
\caption{Explaining change in abstract argumentation.}
\label{fig:abstract-arg}
\end{figure}

\subsection{Related Research}
\label{related}
The work presented in this paper adds to a comprehensive body of research of explainability in formal argumentation and, more broadly, symbolic AI.
However, let us highlight that our work is novel in two key aspects, as further explained below:
\begin{itemize}
    \item It is among the first works focusing on explainability in quantitative bipolar argumentation;
    \item More broadly, it is the first work focusing on explaining change of inference in formal argumentation -- colloquially, answering ``Why A and no longer B''? instead of simply ``Why B?'' -- that is not limited to a particular set of argumentation semantics (in contrast to our earlier work~\cite{clarxai}).
\end{itemize}

Our notion of sufficiency corresponds to the familiar notion in other explainable AI works, such as~\cite{10.1007/978-3-030-86772-0_4,Darwiche_Ji_2022}. 
Yet, our notion of necessity differs somewhat from a similar notion as presented e.g.\ in~\cite{10.1007/978-3-030-86772-0_4} (though dynamic reasoning scenarios are not considered therein).
In our case, change of inference does not imply that \emph{all} arguments contained in a given necessary explanation need to have been changed. Instead, the change of inference implies that \emph{some} argument(s) in the explanation must have been changed. 
This is more in line with the concept of necessary reasons introduced in \cite{Darwiche_Ji_2022} (though not concerning argumentation at all). 
Finally, in principle, a counterfactual as an explanation~\cite{Ginsberg:1986} could just be a (minimal) change to the current world (updated QBAF in our case) such that the inference is different (in our case, \sconsy\ preserved). 
We additionally require sufficiency, lest a counterfactual explanation be merely a set arguments without changes to which the inference would be different but which do not really result into the observed inference (poetically, ``the straw that broke the camel's neck''). 
As an explanation, a sufficient counterfactual (as in our case) feels more informative. 

From the point of view of explainability, to our knowledge this work is the first paper focusing on quantitative bipolar argumentation. 
Even so, our explanations are still immediately applicable to quantitative (abstract) argumentation, where explainability has not been researched either, with the exception of~\cite{Delobelle:Villata:2019}. 
There, the authors formalise a notion of \emph{impact} of an argument on the final strength of another argument, roughly as a difference between the final strengths of the latter argument with and without the former argument being present. 
We instead consider as explanations the changes to arguments that guarantee alterations in the relative strengths of other arguments after a given update to the quantitative argumentation framework. 

More generally, our work is positioned at the intersection of argumentation dynamics and explainable argumentation, both of which have been studied in depth: see~\cite{doutre-argument} for a survey on argumentation dynamics, as well as~\cite{vassiliades_bassiliades_patkos_2021} and~\cite{Cyras.et.al:2021-IJCAI} for surveys on argumentation and explainability. 
Few works study the intersection of dynamics and explainability \emph{explicitly}. 
A notable exception is~\cite{clarxai}, where we studied, in the context of (admissibility-based) abstract argumentation, how the violation of monotony of entailment can be explained in so-called \emph{normal expansion} scenarios, in which new arguments are added to an argumentation framework, but the relation among previously existing arguments remains unchanged. 
The present work is different in that it 
\begin{enumerate*}[i)]
\item 
addresses QBAFs, and
\item explains \nsconsy\ (i.e.\ change in preferences from a decision-theoretical perspective) rather than the violation of monotony of entailment.
\end{enumerate*}

Several argumentation explainability approaches consider dynamics \emph{implicitly}. 
For instance, assuming some space of modifications in a given argumentation framework, the modifications that would change some topic argument's acceptability status (or strength) can be seen as explanations of such a change~\cite{Wakaki:Nitta:Sawamura:2009,Booth.et.al:2014,Sakama:2018}.
In particular, a collection of additions or removals of  arguments or attacks in an abstract argumentation framework in a way that changes the acceptability of a specific argument is an explanation in e.g.~\cite{Fan:Toni:2015-TAFA,Sakama:2018}. 
Relatedly, though not directly concerning changes, 
\cite{Saribatur.et.al:2020,Ulbricht:Wallner:2021,10.1007/978-3-030-86772-0_4} define explanations, roughly speaking, as sets of arguments (in non-quantitative argumentation frameworks) that are sufficient for acceptance or rejection of some target argument(s). 
Note, however, that in this paper we do not explicitly introduce QBAF update operations, such as adding~\cite{Cayrol2010} or removing an argument~\cite{Bisquert2011}. 
While we have motivated why we focus on changes only with respect to arguments in the abstract, it is an interesting future work direction to define explanations that consist of not only arguments (as in this paper) but also of the exact changes with respect to those arguments.

The works mentioned above concerning argumentation dynamics focus on non-quantitative (non-numerical) argumentation frameworks. We instead consider QBAFs in this paper, focusing on gradual semantics and changes to numerical argument strengths. 
Nonetheless, in the future one can expand the current perspective on QBAF (change) explainability by in addition providing sub-graphs to trace sets of explanation arguments to topic arguments. Also, one could further extend the notion of \nsconsy\ by stipulating that a threshold in the difference in strengths needs to be exceeded. This threshold could be either relative, e.g., the initially weaker argument needs to be 5\% stronger (in the update) than the initially stronger argument; or absolute, e.g. the initially weaker argument needs to be stronger by $x \in \interval$ (in the update) than the initially stronger argument. Here, we conjecture that the formal results with respect to soundness and completeness would apply, but formal analysis is still to be done.

\section{Concluding Remarks}
\label{conclusion}
In this paper, we introduced explanations for changes in the relative strengths of two topic arguments after a QBAF update. Explanations are in the form of sets of arguments that have been changed: added, removed, modified in their initial strength or outgoing attacks and supports.
Intuitively,
changes with respect to the arguments in the explanations are responsible for the change in inference, namely the relative change in the final strengths of topic arguments, in the evolving QBAF. 
Our approach thus helps to answer a key explainability question  -- ``why $\argb$ and no longer $\arga$?'' -- in dynamic quantitative bipolar argumentation.
We answer this question using the following aspects of explanations.
\begin{enumerate}
    \item Sufficiency: changes to which arguments suffice to imply a change of inference in the absence of other changes?
    \item Counterfactuality: changes to which arguments both suffice for the inference in the absence of other changes, and if reversed, restore the initial inference when all other changes stay?
    \item Necessity: changes to which arguments are necessarily implied by the change of inference.
\end{enumerate}

We believe to have contributed with novel forms of explanations of inference in dynamically evolving QBAFs, employing the notions of sufficiency, counterfactuality and necessity. 
By means of change reversals in evolving QBAFs, we provided formal definitions of explanations as sets of arguments whose change is responsible for the change in the inference. 
We supplied basic theoretical analysis of soundness and completeness of our explanations as well as high-level algorithmic details of an implementation for generating explanations. 
We hope our work meaningfully complements the research on explainability in computational argumentation specifically and explainable AI at large.

\subsubsection*{Acknowledgements}
We thank the reviewers of this paper, as well as the reviewers of the initial conference paper version, for their substantial and very useful feedback. Example~\ref{ex:reviewer} was provided by one of the anonymous reviewers.

\bibliographystyle{elsarticle-num}

%

\end{document}